\documentclass[twocolumn]{IEEEtran}

\usepackage{amsmath,amssymb,amsfonts}
\usepackage{amsthm}
\usepackage[linesnumbered,lined,boxed,commentsnumbered,ruled]{algorithm2e}
\usepackage{graphicx}
\allowdisplaybreaks
\usepackage{tikz}
\usepackage{soul}
\usepackage{makecell}
\usepackage{booktabs}
\usepackage{multirow}
\usepackage{colortbl}
\usepackage{pgfplots}
\usepackage{caption}
\usepackage{subcaption}

\usepackage{bm}
\pgfplotsset{compat=newest}
\newcommand{\tabl}[2]{\begin{tabular}{#1} #2 \end{tabular}}

\newtheorem{theorem}{Theorem}

\newtheorem{remark}{Remark}

\newtheorem{lemma}[theorem]{Lemma}
\usepackage{chngcntr}
\newtheorem{assumption}{}

\newcommand{\R}{\mathbb{R}}
\newcommand{\E}{\mathbb{E}}
\newcommand{\tr}{^\mathsf{\scriptscriptstyle T}}
\newcommand{\F}{\mathcal{F}}
\newcommand{\N}{\mathcal{N}}
\newcommand{\K}{\mathcal{K}}

\usepackage{todonotes}
\usepackage{xspace}

\begin{document}
\title{Generalized Simultaneous Perturbation-based Gradient Search with Reduced Estimator Bias}
\author{
	Soumen Pachal$^\dagger$,  Shalabh Bhatnagar$^\ddag$ and Prashanth L.A.$^\star$
	\thanks{ 
		\hspace{-1.5em}$\dagger$ Department of Computer Science and Engineering,
		Indian Institute of Technology Madras, Chennai; TCS Research, India,
		E-Mail: cs22d009@smail.iitm.ac.in,	\\
		$\ddag$ Department of Computer Science and Automation and the Robert Bosch Centre for Cyber Physical Systems,
		Indian Institute of Science, Bangalore,
		E-Mail: shalabh@iisc.ac.in,\\
		$\star$ Department of Computer Science and Engineering,
		Indian Institute of Technology Madras, Chennai,
		E-Mail: prashla@cse.iitm.ac.in.\\	
	}
}

\maketitle
\begin{abstract}
We present in this paper a family of generalized simultaneous perturbation-based gradient search (GSPGS) estimators that use noisy function measurements. The number of function measurements required by each estimator is guided by the desired level of accuracy. We first present in detail unbalanced generalized simultaneous perturbation stochastic approximation (GSPSA) estimators and later present the balanced versions (B-GSPSA) of these. We extend this idea further and present the generalized smoothed functional (GSF) and generalized random directions stochastic approximation (GRDSA) estimators, respectively, as well as their balanced variants. We show that estimators within any specified class requiring more number of function measurements result in lower estimator bias. We present a detailed analysis of both the asymptotic and non-asymptotic convergence of the resulting stochastic approximation schemes. We further present a series of experimental results with the various GSPGS estimators on the Rastrigin and quadratic function objectives. Our experiments are seen to validate our theoretical findings. 
\end{abstract}


\section{Introduction}

This paper deals with the problem of stochastic optimization. Let $f: \R^d\times \R^k \rightarrow \R$ denote a general non-linear and noisy performance function that is a function of a parameter $\theta\in\R^d$ and a random vector $\xi\in\R^k$. The goal is to find a parameter $\theta^* \in \R^d$ that minimizes the objective function $F(\theta) = E_\xi[f(\theta,\xi)]$ over all $\theta \in \R^d$. In other words, we wish to find $\theta^*\in\R^d$ for which
\begin{equation}
\label{F-eq}
F(\theta^*) = \min_{\theta\in\R^d} F(\theta).
\end{equation}

We assume that the analytical form of the function $F$ is not known but noise-corrupted samples $f(\theta,\xi)$ of the performance objective are available. Since finding a global minimum of a function such as $F(\cdot)$ is difficult in general, one often settles for the less-difficult problem of finding a local minimum of such an objective. Infinitesimal perturbation analysis (IPA) \cite{fu2015handbook}, \cite{hocao} is a class of techniques that incorporate stochastic gradient search though under strong requirements on the sample performance objective. They work in situations where the gradient and expectation operators can be interchanged, i.e., that $\nabla F(\theta) = E_\xi[\nabla f(\theta,\xi)]$. This is particularly difficult to achieve when the noise $\xi$ is also a function of the parameter $\theta$. Nonetheless when such conditions hold, IPA methods are efficient and require only a single function measurement at each instant and are seen to converge to local minima.  

In situations where IPA type approaches do not work, one resorts to gradient estimation techniques that are typically zeroth-order methods where the gradient estimators are constructed from noisy sample performance measurements obtained at certain perturbed parameter values. The earliest such approach is due to Kiefer and Wolfowitz \cite{kw} which was originally proposed for the case when the parameter is a scalar and requires two function measurements. When translated to the case of a vector ($d$-dimensional) parameter, this approach estimates the gradient by using $2d$ function measurements, two for each of the $d$ partial derivatives. The gradient estimator in this case has  the form $\hat{\nabla} F(\theta) = (\hat{\nabla}_1F(\theta),\hat{\nabla}_2F(\theta),\ldots,\hat{\nabla}_dF(\theta))^T$, where
\begin{equation}
\label{KWe}
\hat{\nabla}_i F(\theta) = \frac{f(\theta+\delta e_i,\xi_i^+)-f(\theta-\delta e_i,\xi_i^-)}{2\delta},
\end{equation}
$i=1,\ldots,d$. Here $e_i$ is the unit vector having 1 in the $i$th place and 0's elsewhere, while $\xi_i^+,\xi_i^-$ denote the noise random variables arising from the function measurements at parameters $\theta +\delta e_i$ and $\theta - \delta e_i$, respectively. 
Here $\delta>0$ is a small sensitivity parameter and one requires $\delta\rightarrow 0$ slowly enough for convergence to a local minimum. On the other hand, if one is satisfied with convergence to a small ($\epsilon$) neighborhood of a local minimum, then a small enough (though constant) $\delta$ would suffice. The approximate stochastic gradient scheme with noisy gradient estimators is the following:
\begin{equation}
\label{RM}
\theta(n+1) = \theta(n) -a(n) \hat{\nabla}F(\theta(n)),
\end{equation}
starting with some $\theta(0)\in \R^d$. Such an algorithm would fall under the broad category of stochastic approximation algorithms \cite{rm} whose averaging properties help achieve the desired objective of convergence to a local minimum or a small neighborhood of it depending (as alluded to above) on whether $\delta\rightarrow 0$ or is held fixed to a small enough value. 

Algorithm (\ref{RM}) together with the estimator (\ref{KWe}) is also referred to as the finite difference stochastic approximation (FDSA) or the Kiefer-Wolfowitz (K-W) scheme.
One of the key disadvantages of the K-W scheme is the large number ($2d$) of simulations that it requires for a $d$-dimensional parameter. The computational effort required for this scheme is thus significant particularly when $d$ is large.

Over a number of years, it has been observed that approaches based on random (simultaneous) perturbations such as the smoothed functional scheme \cite{katkul}, \cite{rubinstein}, \cite{bhatbor3}, \cite{bhat2}, the random directions stochastic approximation \cite{kushcla}, \cite{prashanth2017rdsa}, \cite{prashanth2020rdsa}, \cite{mondal2022} and the 
simultaneous perturbation stochastic approximation (SPSA) \cite{spall1992multivariate, spall1997one} and it's deterministic perturbation variant \cite{bhatfumarcwang} are more efficient than the K-W procedure. The broad idea in these approaches is to use only a limited number of function measurements (often one or two) to get estimates of all the partial derivatives of the performance objective. This is achieved by perturbing all the component directions most often with distributions that satisfy certain desired properties. For instance, distributions such as the Gaussian \cite{bhatbor3, bhat2}, Cauchy \cite{mondal2022}, uniform \cite{prashanth2017rdsa} and even q-Gaussian \cite{DeBhAd1,DeBhAd2} have been found to work efficiently.

Amongst the most popular of this class of the simultaneous perturbation approaches is SPSA \cite{spall1992multivariate}. It is hugely popular because it requires only two function measurements for any $d$-dimensional parameter. Here one perturbs all the parameter component directions using random variates that require properties that are most commonly satisfied by independent, symmetric, zero-mean Bernoulli distributed random variables. The two-measurenent SPSA gradient estimate proposed in \cite{spall1992multivariate} has been found to be effective and takes the following form:
\begin{equation}
\label{SPSAe}
\hat{\nabla}_i F(\theta) = \frac{f(\theta+\delta \Delta)-f(\theta-\delta \Delta)}{2\delta\Delta_i},
\end{equation}
where $\Delta = (\Delta_1,\ldots,\Delta_d)^T$, with $\Delta_i$ being independent of $\Delta_j$ for all $i\not=j$ and with $\Delta_i =\pm 1$ with probability $1/2$, $\forall i=1,\ldots,d$. It can be shown using suitable Taylor's expansions, see \cite{spall1992multivariate}, that the conditional expectation of $\hat{\nabla}_i F(\theta)$ given $\theta$ in (\ref{SPSAe}) gives the $i$th partial derivative $\nabla_i F(\theta)$ plus a bias term that is of order $o(\delta)$. Thus, when this gradient estimator is used in a stochastic approximation update as in (\ref{RM}), one obtains from the averaging properties of stochastic approximation, convergence to a local minimum of $F$ provided $\delta\rightarrow 0$ slowly enough.

The goal of this paper is to generalize the simultaneous perturbation based gradient search procedures as above by giving a family of these where the number of function measurements required in the estimator is guided by the desired level of accuracy. For better exposition, we focus for the large part on the generalization of the gradient SPSA procedure comprising of unbalanced estimators and later present similar generalizations for balanced SPSA estimators as well as other procedures such as generalized smoothed functional (GSF) and generalized random directions stochastic approximation (GRDSA) respectively. 

Our starting point in this paper is Chapter VII.1a of \cite{asmussen} where the case of scalar parameter $\theta$ is considered and it is shown using suitable Taylor's expansions that for a bias in the estimator of order $\delta^k$, one can obtain a suitable simulation-based estimator to the gradient, for any $k>1$. This idea has also been explored under a scalar parameterization in \cite{fu2020differentiation}.
We extend this idea to the case of vector parameters by incorporating the simultaneous perturbation idea and in the process obtain a family of gradient SPSA estimators that involve an increasing number of function measurements depending on the desired level of accuracy. 

We consider the case of $\delta\rightarrow 0$ for our asymptotic analysis as with many other previous works, cf.~\cite{kw, spall1992multivariate, prashanth2017rdsa} where the algorithm is shown to converge to the stable equilibria of an ordinary differential equation (ODE). An alternative that we didn't pursue is the case when $\delta>0$ is a constant. In this case, one may show as in \cite{arun-bhatnagar2}, that the algorithm tracks the attractors of a differential inclusion. The precise requirement on $\delta$ is captured in Assumption (A4) of Section~\ref{convergence}. We also provide a non-asymptotic convergence result in addition to asymptotic convergence.

The rest of the paper is organized as follows: In Section~\ref{GSPSA}, we present the Generalized SPSA gradient operator that is seen to result in the Generalized SPSA gradient estimators involving  function measurements with various perturbed parameters. In Section~\ref{123estimators}, we show the precise form of the Generalized (unbalanced) SPSA estimators for the first few cases as well as a $(k_1+1)$ measurements estimator, $k_1 \geq 1$, that we show through suitable Taylor's expansions, provides an estimator bias of order $O(\delta^{k_1})$. 

In Section \ref{sec: balanced_estimator12}, we present the Balanced Generalized SPSA estimators for the first few cases as well as a general such estimator. In Section~\ref{convergence}, we provide the main theoretical results associated with the proposed scheme. In Section \ref{sec:unified_estimator}, we focus on the generalization to the other simultaneous perturbation-based estimators. The proofs of these main results are then presented in Section~\ref{proofs}. 

In Section \ref{sec:experiments}, we present the results of several simulation experiments which are seen to validate the theoretical findings. In particular, we observe that for a given simulation budget, estimators that require more number of function measurements within a prescribed class, in general, result in better error performance. Finally, in Section~\ref{conclusions}, we present the conclusions of this study and also present some directions for further research. 

A preliminary version of this paper without the balanced generalized SPSA estimators as well as generalizations of the other gradient estimators such as smoothed functional (SF), random directions stochastic approximation (RDSA) etc.,
was published in \cite{bhatnagar2023generalized}.  Moreover, the convergence analysis in \cite{bhatnagar2023generalized} was performed for Generalized SPSA with $k_1=3$. In contrast, we provide a generalized convergence proof that applies for any $k_1\ge 1$, as well other generalized estimators such as SF, RDSA, etc. In addition, we establish convergence guarantees for the Balanced Generalized SPSA algorithm, which in particular shows a better convergence rate. Finally, unlike \cite{bhatnagar2023generalized}, we numerically validate our proposed algorithms.

\section{Generalized SPSA-Based Gradient Operator}
\label{GSPSA}
The key idea here is to construct finite difference estimators of $\nabla F(\theta)$ for any given order of the bias. In Chapter VII.1a of \cite{asmussen}, this idea has been explored in the context of scalar functions $F:\mathbb{R}\rightarrow\mathbb{R}$. A regular Taylor's expansion in terms of the differentiation operator is performed in order to obtain generalized finite difference estimates of the derivative of $F$. 
We extend this idea to the case of vector-valued parameters $\theta\in\mathbb{R}^d$ and for functions $F:\mathbb{R}^d\rightarrow\mathbb{R}$ by invoking the multi-variate version of the Taylor series expansion. In particular, as with SPSA, we select the directional perturbation vector in the Taylor's expansion as the vector $\Delta=(\Delta_1,\Delta_2,\ldots,\Delta_d)^T$ of independent and symmetric Bernoulli distributed random variables $\Delta_i = \pm 1$ with probability $1/2$.

Define the differentiation operator ${\cal D}^\beta$ as
\[{\displaystyle {\cal D}^\beta F(\theta) \triangleq \frac{\partial^{|\beta|}F(\theta)}{\partial\theta_1^{\beta_1}\cdots \partial\theta_d^{\beta_d}}},\]
with $|\beta|=\beta_1+\cdots+\beta_d$ and $\theta^\beta=\theta_1^{\beta_1}\cdots\theta_d^{\beta_d}$, for $\beta_1,\cdots,\beta_d \geq 0$. Further, $\beta! =\beta_1!\beta_2!\cdots\beta_d!$.
The multi-variate Taylor's expansion has the following form: 
\begin{equation}
\label{mte}
F(\theta+\delta\Delta) = \sum_{|\beta|=0}^{\infty} \frac{{\cal D}^\beta F(\theta)}{\beta!}(\delta\Delta)^\beta
= \sum_{|\beta|=0}^{\infty} \left(\frac{(\delta\Delta{\cal D})^\beta}{\beta!}\right)F(\theta),
\end{equation}
assuming that $F$ is infinitely many times continuously differentiable. One may now define a shift operator $\tau_{\delta\Delta}$ as follows:
$\tau_{\delta\Delta}F(\theta) \equiv F(\theta+\delta\Delta)$. This allows us to rewrite \eqref{mte} as 
\[
\tau_{\delta\Delta} = \exp(\delta\Delta{\cal D}),
\]
which implies
\[
{\cal D} = \frac{1}{\delta\Delta}\log(\tau_{\delta\Delta}),
\]
where 
${\displaystyle 
\frac{1}{\delta\Delta} \triangleq \left(\frac{1}{\delta\Delta_1},\ldots,\frac{1}{\delta\Delta_d}\right)^T.
}$
An expansion of the $\log$ function then would give
\[
{\cal D} = \frac{1}{\delta\Delta}
\sum_{j=1}^{\infty}\frac{(\tau_{\delta\Delta} - {\cal I})^j}{j}(-1)^{j+1},\]
where ${\cal I}$ denotes the identity operator, and 
$\tau_{\delta\Delta}^{k} = \tau_{k \delta\Delta}$, for any $k$. 
We can view our generalized gradient operator as follows: Let ${\cal D} = ({\cal D}_i, i=1,\ldots,d)^T$ where for $i=1,\ldots,d$,
\begin{equation} 
\label{gestimator}
{\cal D}_i = \frac{1}{\delta\Delta_i} \sum_{j=1}^{\infty}\frac{(\tau_{\delta\Delta} - {\cal I})^j}{j}(-1)^{j+1}.
\end{equation}
The above gradient operator ${\cal D}$ is an exact (or ideal) operator that will result in no bias in the gradient estimation procedure. 
In the next section, we propose to truncate \eqref{gestimator}
by only taking say the first $k_1$ terms in the summation for some $k_1\geq 1$ while ignoring the remaining terms. We show that this results in powerful SPSA-based gradient estimators that require $(k_1+1)$ function measurements and provide a bias of $O(\delta^{k_1})$. Thus, for a small enough  $\delta$, by increasing $k_1$ and thereby also increasing the number of function measurements as required by the resulting gradient estimator, we show using Taylor's expansions that one may significantly reduce the estimator bias. 

\section{Generalized (Unbalanced) Gradient SPSA Estimators}
\label{123estimators}
We present in this section a string of GSPSA algorithms that are obtained by truncating the series in (\ref{gestimator}) and also provide a result on the bias in these algorithms. As discussed towards the end of the previous section, from the form of the gradient operator in (\ref{gestimator}), one can obtain a $(k_{1}+1)$ measurements estimator by truncating the series above at $k_{1}$ (i.e., by only taking the sum of the first $k_1$ terms in the summation in \eqref{gestimator}). 

We present below the form of the GSPSA estimators based on the first few measurements and briefly analyze these for improved clarity. For ease of exposition, we ignore the noise in these estimators that involve two, three, four and five measurements, respectively, of the performance objective. However, we do account for noise in the generalized $(k_1+1)$ measurements estimator, which is presented in \eqref{eq:gspsa-est} after the illustrative cases corresponding to $k_1=1,\ldots,4$. 

\subsection{Two measurements GSPSA}
The two measurements version of SPSA when using the GSPSA estimator (\ref{gestimator}) will correspond to truncating the series there at $j=1$ (i.e., only considering the first term). Then, we will have
\[
{\cal D}^1_iF(\theta) = \left(\frac{\tau_{\delta\Delta}-{\cal I}}{\delta\Delta_i} \right) F(\theta)
= \frac{F(\theta+\delta\Delta)-F(\theta)}{\delta\Delta_i},
\]
where ${\cal D}^1 \triangleq ({\cal D}^1_i, i=1,\ldots,d)^T$ denotes the first order approximation operator. 
This is the one-sided version of SPSA, and has been analyzed in \cite{chen1999kiefer} for its convergence properties. In particular, performing a Taylor's expansion of $F(\theta+\delta\Delta)$ in the above, one obtains that
\begin{equation}
\label{1spsa}
{\cal D}^1_iF(\theta)=\frac{F(\theta+\delta\Delta)-F(\theta)}{\delta\Delta_i}
= \frac{\Delta^T\nabla F(\theta)}{\Delta_i} + O(\delta).
\end{equation}

\subsection{Three measurements GSPSA}

This estimator will be obtained from (\ref{gestimator}) by truncating the series there at $j=2$. Thus, we have in this case
\begin{align*}
{\cal D}^2_iF(\theta) &= \left[ \left(\frac{\tau_{\delta\Delta}-{\cal I}}{\delta\Delta_i}\right) - \frac{(\tau_{\delta\Delta}-{\cal I})^2}{2\delta\Delta_i}\right]F(\theta)\\
&= \left[\left(\frac{\tau_{\delta\Delta}-{\cal I}}{\delta\Delta_i}\right)
- \left( \frac{\tau_{2\delta\Delta}+{\cal I} -2\tau_{\delta\Delta}}{2\delta\Delta_i}
\right)\right]F(\theta) \\
&= \left(\frac{F(\theta+\delta\Delta)-F(\theta)}{\delta\Delta_i}\right) \\
& \quad- \left(\frac{F(\theta+2\delta\Delta) + F(\theta) - 2F(\theta+\delta\Delta)}{2\delta\Delta_i}
\right)\\
&= \left(\frac{4F(\theta+\delta\Delta)-3F(\theta)-F(\theta+2\delta\Delta)}{2\delta\Delta_i}\right).
\end{align*}
As before, ${\cal D}^2$ indicates the second order approximation operator.
Like the estimators that follow, this is a new gradient SPSA estimator that has previously not been presented. Upon performing a Taylor's expansion of the perturbed quantities in the final expression in the above equation, we obtain 
\[
F(\theta+\delta\Delta) = F(\theta) + \delta \Delta^T \nabla F(\theta) +\frac{\delta^2}{2}\Delta^T\nabla^2 F(\theta)\Delta + O(\delta^3).
\]
Similarly,
\[
F(\theta+2\delta\Delta) \!=\! F(\theta) +2\delta \Delta^T \nabla F(\theta) + \frac{4\delta^2\Delta^T\nabla^2F(\theta)\Delta}{2} + O(\delta^3).
\]
Thus, upon simplification, one obtains
\[{\cal D}^2_iF(\theta) = 
\left(\frac{4F(\theta+\delta\Delta)-3F(\theta)-F(\theta+2\delta\Delta)}{2\delta\Delta_i}\right)\]
\begin{equation}
\label{2spsa}
= \frac{\Delta^T \nabla F(\theta)}{\Delta_i} + O(\delta^2).
\end{equation}
The first term in the expansion in (\ref{2spsa}) is the same as the first term in the RHS of  (\ref{1spsa}). However, the second term in (\ref{2spsa}) is $O(\delta^2)$ as opposed to $O(\delta)$ in (\ref{1spsa}). It is interesting however to note here that the balanced two-simulation SPSA estimator of \cite{spall1992multivariate}, see (\ref{SPSAe}), has a similar Taylor's expansion as (\ref{2spsa}).

\subsection{Four measurements GSPSA}
The estimator here is obtained from (\ref{gestimator}) by truncating the series at $j=3$. Thus, we have
\begin{align*}
&{\cal D}^3_iF(\theta) \\
&= \left[ \left(\frac{\tau_{\delta\Delta}-{\cal I}}{\delta\Delta_i}\right) - \frac{(\tau_{\delta\Delta}-{\cal I})^2}{2\delta\Delta_i}
+ \frac{(\tau_{\delta\Delta}-{\cal I})^3}{3\delta\Delta_i}
\right]F(\theta)\\
&= \Bigg[\left(\frac{\tau_{\delta\Delta}-{\cal I}}{\delta\Delta_i}\right)
- \left( \frac{\tau_{2\delta\Delta}+{\cal I} -2\tau_{\delta\Delta}}{2\delta\Delta_i}
\right)\\
&\qquad+ \left(\frac{\tau_{3\delta\Delta} -3\tau_{2\delta\Delta} + 3\tau_{\delta\Delta} -{\cal I}}{3\delta\Delta_i}
\right)\Bigg]F(\theta)\\
&= \frac{2F(\theta+3\delta\Delta) - 9 F(\theta+2\delta\Delta)+ 18F(\theta+\delta\Delta) -11F(\theta)}{6\delta\Delta_i}.
\end{align*}
The last equality above is obtained upon simplification. Now Taylor's expansions as before give us 

\begin{equation} 
\label{3spsa}
{\cal D}^3_iF(\theta) = \frac{\Delta^T \nabla F(\theta)}{\Delta_i}  + O(\delta^3).
\end{equation}
Note here that the zeroth order as well as the second and third order terms turn out to be exactly equal to zero due to cancellations of the various terms in the expansions resulting in \eqref{3spsa}.  

\subsection{Five measurements GSPSA}
\label{sec:4spsa}
Here the estimator is obtained from (\ref{gestimator}) by truncating the series at $j=4$. Thus, the expression for ${\cal D}^4_i F(\theta)$ will be the following:
\begin{align*}
&{\cal D}^4_iF(\theta)\\
&= \Bigg[\frac{(\tau_{\delta\Delta}-{\cal I})}{\delta\Delta_i} - \frac{(\tau_{\delta\Delta}-{\cal I})^2}{2\delta\Delta_i}
+ \frac{(\tau_{\delta\Delta}-{\cal I})^3}{3\delta\Delta_i}\\
&\qquad - \frac{(\tau_{\delta\Delta}-{\cal I})^4}{4\delta\Delta_i}\Bigg]F(\theta)\\
&= \Big[\left(\frac{\tau_{\delta\Delta}-{\cal I}}{\delta\Delta_i}\right)
- \left( \frac{\tau_{2\delta\Delta}+{\cal I} -2\tau_{\delta\Delta}}{2\delta\Delta_i}
\right) \\
&\qquad+ \left(\frac{\tau_{3\delta\Delta} -3\tau_{2\delta\Delta} + 3\tau_{\delta\Delta} -{\cal I}}{3\delta\Delta_i}
\right)\\
&\qquad- \left(\frac{\tau_{4\delta\Delta} +6\tau_{2\delta\Delta} -4\tau_{3\delta\Delta} -4\tau_{\delta\Delta} +{\cal I}}{4\delta\Delta_i}
\right)\Big]F(\theta)\\
&= \frac{-3F(\theta+4\delta\Delta) + 16F(\theta+3\delta\Delta) -36F(\theta+2\delta\Delta) }{12\delta\Delta_i}\\
&\qquad \qquad + \frac{48 F(\theta+\delta\Delta)-25F(\theta)}{12\delta\Delta_i}.   
\end{align*}
Taylor's expansions as before can be seen to give us
\begin{equation} 
\label{4spsa}
{\cal D}^4_iF(\theta) = \frac{\Delta^T \nabla F(\theta)}{\Delta_i} + O(\delta^4).
\end{equation}
Note here again that the zeroth order as well as the second, third and fourth order terms turn out to be exactly zero due to suitable cancellations of the various terms in the Taylor's expansions thereby resulting in the form \eqref{4spsa} of the gradient estimator above. 

\subsection{$(k_1+1)$ measurements GSPSA }
\label{sec:kspsa}
Here the estimator is obtained from (\ref{gestimator}) by truncating the series for a general $k_1\geq 1$. Thus, the expression for the operator ${\cal D}^{k_1}_i$ will be the following:
\begin{align*}
{\cal D}^{k_1}_i &= \frac{1}{\delta\Delta_i} \sum_{j=1}^{k_1}\frac{(\tau_{\delta\Delta} - {\cal I})^j}{j}(-1)^{j+1}\\
& = \frac{1}{\delta\Delta_i} \sum_{j=1}^{k_1} \frac{(-1)^{j+1} }{j} \sum_{l=0}^j {j \choose l} (\tau_{\delta\Delta} )^l (-1)^{j-l}\\
& = \frac{1}{\delta\Delta_i} \sum_{l=0}^{k_1} \frac{(-1)^{1-l} (\tau_{\delta\Delta})^l}{l!} \sum_{j=l}^{k_1} \frac{(j-1)!}{(j-l)!}\\
& = \frac{1}{\delta\Delta_i} \sum_{l=0}^{k_1} \frac{(-1)^{1-l} (\tau_{\delta\Delta})^l}{l!} C^{k_1}_l,
\end{align*}

where $C^{k_1}_l=
\begin{cases}
\frac{1}{l}\prod\limits_{j=0}^{l-1} (k_1-j),& l \geq 1, \\
\sum_{j = 1}^{k_1} \frac{1}{j}, & l = 0.
\end{cases}
$

Thus, we have
\begin{equation}
\label{eq:gspsa_noiseless_est}
    \begin{aligned}
{\cal D}^{k_1}_iF(\theta) 
& = \left[\frac{1}{\delta\Delta_i} \sum_{l=0}^{k_1} \frac{(-1)^{1-l} C^{k_1}_l\tau_{l\delta\Delta}}{l!}\right] F(\theta) \\
&= \frac{1}{\delta\Delta_i} \sum_{l=0}^{k_1} \frac{(-1)^{1-l} C^{k_1}_l F(\theta+l\delta\Delta)}{l!}.
    \end{aligned}
\end{equation}
In the general case with noisy observations, the generalized SPSA estimator will have the following form:
\begin{align}
&
\widehat{\cal D}^{k_1}_iF(\theta(n))\triangleq\nonumber\\ 
& \frac{1}{\delta(n)\Delta_i(n)} \sum_{l=0}^{k_1} \frac{(-1)^{1-l} C^{k_1}_l \{ f(\theta(n)+l\delta(n)\Delta(n), \xi_l(n)) \}}{l!},\label{eq:gspsa-est}
\end{align}
 where $\{\xi_l(n)\}_{l = 0}^{k_1}$ is an independent and identically distributed (i.i.d) noise sequence as the noise in the estimator measurements. Note that one may write
 \begin{align}
     &f(\theta(n) + l \delta(n)\Delta(n), \xi_l(n)) \nonumber\\
     &\quad= F(\theta(n) + l \delta(n)\Delta(n)) + M_{n+1}^l, \nonumber
 \end{align}
 where $M_{n+1}^l = f(\theta(n) + l \delta(n)\Delta(n), \xi_l(n)) - F(\theta(n) + l \delta(n)\Delta(n))$ and $F(\theta(n) + l \delta(n)\Delta(n)) = \E \left[f(\theta(n) + l \delta(n)\Delta(n), \xi_l(n))| \theta(n)\right],$ where the expectation is w.r.t the common distribution of $\xi_l(n)$. Note that with $\mathcal{F}_n^l = \sigma(\theta(m), \Delta(m), m\leq n, \xi_l(m), m < n), n\geq 0, (M_n^l, \mathcal{F}_n^l), n\geq 0$ is a martingale difference sequence. 

\subsection{The $(k_1+1)$-Measurement GSPSA Algorithm}
Algorithm \ref{alg:GSPSA} presents the pseudo-code for a general $(k_1+1)$-measurement GSPSA algorithm. The precise form of the estimator will then depend on the specific value of $k_1$ being used. 
This is a stochastic gradient algorithm, which employs the GSPSA estimator, defined in \eqref{eq:gspsa-est}. 
In this algorithm, 
$\Delta(n), n\geq 0$ is a sequence of independent  vectors of symmetric, independent, $\pm 1$-valued Bernoulli random variables, while $\delta(n), n\geq 0$ is a time-dependent perturbation sequence that is assumed to be asymptotically vanishing and satisfies \ref{ass:step_size} below. 

\begin{algorithm}
\SetKwInOut{Input}{Input}\SetKwInOut{Output}{Output}
\Input{initial point $\theta(0)$, 
sensitivity parameters $\{\delta(n)\}$, step sizes $\{a(n)\}$, measurements $(k_1+1), k_1\geq 1$, \# iterations $T$.}
\For{$n\leftarrow 0$ \KwTo $T-1$}{
\tcc{Random perturbation}
Generate $\Delta(n)$ using symmetric $\pm1$-valued Bernoulli distribution;

\tcc{Function measurements}
Obtain $\{f(\theta(n) + l \delta(n)\Delta(n),\xi_l(n))\}$, for $l = 0,\ldots,k_1$;

\tcc{Gradient estimation}
Form $\widehat{\cal D}^{k_1}_iF(\theta(n))$ as follows:

$\widehat{\cal D}^{k_1}_iF(\theta(n)) =
 \frac{\Delta_i(n)}{\delta(n)} \sum_{l=0}^{k_1} \frac{(-1)^{1-l} C_l^{k_1} f(\theta(n)+l\delta(n)\Delta(n),\xi_l(n))}{l!}$;

\tcc{Gradient descent}
Perform the following update iteration:
\begin{align}
    &\theta_i(n+1) = \theta_i(n) - a(n) \widehat{\cal D}^{k_1}_iF(\theta(n)).\label{eq:gspsa-gd-update}
\end{align}
 
  }
 \Output{Parameter $\theta(T)$} 
\caption{The GSPSA Algorithm}
\label{alg:GSPSA}
\end{algorithm}

We analyze the asymptotic as well as non-asymptotic performance of a stochastic gradient algorithm using the gradient estimator \eqref{eq:gspsa-est} in Section \ref{convergence}. A crucial ingredient for these analyses is the bias in the gradient estimator, for which we give a bound in the next section.

\subsection{Bias in the GSPSA gradient estimator}
For the analysis of the bias in the general $(k_1+1)$-measurement GSPSA gradient estimator \eqref{eq:gspsa-est},  we make the following assumptions: 

\vspace*{6pt}
\begin{assumption}
    \label{ass:derivative-bound}
    $F: \mathbb{R}^d \rightarrow \mathbb{R}$ is $k_1$-times continuously differentiable with a bounded $(k_1 + 1)$-th derivative.
\end{assumption}

\vspace*{4pt}
\begin{assumption}
\label{ass: del iid}
    $\Delta_i(n), i=1,\ldots,d$ are i.~i.~d.~ random variables that are independent of $\F_n=\sigma(\theta(j), j\leq n, \Delta
    (j),j<n,\xi_0(n),\ldots, \xi_{k_1}(n)), n\geq 1$.
\end{assumption}
\vspace*{4pt}

\begin{assumption}
\label{ass:f bound}
    For some $\alpha_1, \alpha_2 >0$ and for $n\ge 1$, 
 $\E \left[F(\theta(n)+ l\delta(n) \Delta(n))^{2}|\theta(n)\right] \le \alpha_1$ and   $\E[\xi_l(n)^2] \leq \alpha_2$, for  $l=0,\ldots,k_1$. Here the expectation is w.r.t the distribution of $\Delta(n)$. 
\end{assumption}

\vspace*{6pt}
The result below shows that the bias of the $(k_1+1)$ measurements estimator \eqref{eq:gspsa-est} is $O\left(\delta^{k_1}\right)$.

\vspace*{6pt}
\begin{lemma}[Bias lemma]
\label{lemma:gspsa-bias}
\ \\Under \ref{ass:derivative-bound} - \ref{ass:f bound}, for $\widehat{\cal D}^{k_1}_iF(\theta(n)) $ defined according to \eqref{eq:gspsa-est}
we have almost surely (a.s.) that
\begin{align}
& \left| \E\left[\left.\widehat{\cal D}^{k_1}_i F(\theta(n)) \right| \mathcal{F}_n \right] - \nabla_i F(\theta(n))\right| \le c_1\delta(n)^{k_1},\,\, \text{and} \nonumber\\
& \E \left [\left \|\widehat{\mathcal{D}}^{k_1}F(\theta(n)) - \E \left[ \widehat{\mathcal{D}}^{k_1} F(\theta(n))\right]\right \|^2\right ] \!\leq \!\frac{c_2}{\delta(n)^2},  \nonumber
\end{align} 
for $i=1,\ldots,d$, where $c_1$ and $c_2$ are dimension-dependent constants.
\end{lemma}
\begin{proof}
See Section \ref{sec:proof_gspsa}. 
\end{proof}

A few remarks are in order.
\begin{remark}
From the result above, it is apparent that GSPSA has lower bias as compared to one-sided SPSA, see \cite[Section 5.3]{bhatnagar-book}. However, this low bias comes at the cost of extra function measurements. In practical applications where extra function measurements are not expensive, GSPSA is appealing since the lower bias ensures that the rate of convergence is better than one-sided SPSA (see the non-asymptotic analysis in Section \ref{convergence}). An example of such a setting is off-policy reinforcement learning, where the function measurements are obtained using a single dataset, cf. \cite[Chapter 11]{sutton2018reinforcement},\cite{vijayan2021smoothed}.
\end{remark}
\begin{remark}
It is also important to note that Lemma \ref{lemma:gspsa-bias} requires Assumption \ref{ass:derivative-bound}, namely that the function $F$ is $k_1$-times continuously differentiable. This is however required only to obtain the strict bias bound of $O(\delta(n)^{k_1})$ in Lemma~\ref{lemma:gspsa-bias}. The entire family of GSPSA estimators can however continue to be applied even if the objective function $F$ is `less smooth', for instance, if it is only twice continuously differentiable with a bounded third derivative as is the case with one-sided SPSA, i.e., employing function measurements at $\theta(n) + \delta(n)\Delta(n)$ and $\theta(n)$. 
In such a case, the second order terms from the Taylor's expansion of higher order GSPSA estimators would continue to get cancelled giving an estimator bias of $O(\delta(n))$ like SPSA. On the other hand, note  that if the function $F$ satisfies \ref{ass:derivative-bound}, the aforementioned one-sided SPSA estimator  would still continue to give a bias of $O(\delta(n))$ and not $O(\delta(n)^{k_1})$ that the GSPSA estimator would provide. 
To summarize, GSPSA algorithms with $(k_1+1)$ function measurements where the function $F$ is $l$-times continuously differentiable with a bounded $(l+1)$st derivative, for $1\leq l <k_1$, would result in a bias bound of $O(\delta(n)^{l})$, which is clearly superior to one-sided SPSA for $l \ge 2$.
\end{remark}
\begin{remark}
Similar remarks as above also hold for the other GSPGS estimators (B-GSPSA, GSF, GRDSA) that we present in Sections \ref{sec: balanced_estimator12} and \ref{sec:unified_estimator} below.
\end{remark}

\section{Balanced GSPSA Estimators} 
\label{sec: balanced_estimator12}
In the case of vanilla SPSA, the regular balanced estimator of \cite{spall1992multivariate} has a lower bias of $O(\delta^2)$ as compared to the one-sided variant that has a bias bound of $O(\delta)$. This motivates the balanced extension of the GSPSA estimator \eqref{eq:gspsa-est}, which we present in this section. 
As in the case of unbalanced GSPSA, we employ suitable Taylor series expansions  in terms of the differentiation operator. However, instead of the exponential function earlier, we make use of the hyperbolic sine function to arrive at the balanced GSPSA (B-GSPSA) estimator. 

As before, let $\tau_{\delta \Delta} F(\theta) \equiv F(\theta +\delta \Delta),$ with $\tau_{\delta \Delta}$ denoting the shift operator. 
Now note the following: 
\begin{align*}
    \tau_{\delta \Delta} - \tau_{-\delta \Delta} &= \exp(\delta \Delta \mathcal{D}) - \exp(-\delta \Delta \mathcal{D})\\ 
    &= 2 \sinh({\delta \Delta \mathcal{D}}),
\end{align*}
which implies  
\begin{equation}
    \mathcal{D} = \frac{1}{\delta \Delta} \sinh^{-1} \left (\frac{\tau_{\delta \Delta} - \tau_{-\delta \Delta}}{2} \right ). \nonumber
\end{equation}
By using Maclaurin series expansion of $\sinh^{-1}$, we get
\begin{equation}
    \sinh^{-1}(x) = x + \sum_{j =1}^{\infty} (-1)^{j} \frac{1.3\ldots (2j-1)}{2.4 \ldots 2j} \frac{1}{2j+1} x^{2j+1}.
    \label{eq:sinhinv}
\end{equation} 
Using \eqref{eq:sinhinv}, we obtain
\begin{align}
\label{eq:eq_balanced}
&
     \mathcal{D} 
     =
    \frac{1}{\delta \Delta}\left [\left (\frac{\tau_{\delta \Delta}-\tau_{-\delta \Delta}}{2}\right )\right] \nonumber\\ 
    &+ \frac{1}{\delta\Delta}\left [\sum_{j=1}^{\infty} \frac{(-1)^j (2j)!}{2^{2j}(j!)^2} \frac{1}{2j+1} \left ( \frac{\tau_{\delta \Delta} - \tau_{-\delta \Delta}}{2}\right)^{2j+1}\right].
\end{align}
Using the expansion of the gradient operator above, one can obtain any even order approximation by truncating the Taylor series appropriately. A general B-GSPSA estimator will involve $2k_2$ function measurements for some $k_2\geq 1$. We present a couple of illustrative cases below followed by the general form of the $2k_2$ measurement estimator for any $k_2\geq 1$.

\subsection{Two-measurement B-GSPSA}
The two-measurements version of B-GSPSA is obtained by considering only the first term in \eqref{eq:eq_balanced}, i.e.,
\begin{align}
    \mathcal{D}_i^1 F(\theta) &= \frac{1}{\delta \Delta_i} \left[\frac{\tau_{\delta \Delta}-\tau_{-\delta \Delta}}{2}\right ]F(\theta) \nonumber\\ \nonumber
    &= \frac{1}{\delta \Delta_i}\left[\frac{F(\theta + \delta\Delta) - F(\theta -\delta \Delta)}{2}\right].
\end{align}
The expression above coincides with the classic (two-sided) SPSA \cite{spall1992multivariate}. 
By employing Taylor series expansions of $F(\theta \pm \delta \Delta)$, it is easy to see that
\begin{align}
    \mathcal{D}_i^1F(\theta) &= \frac{1}{\delta \Delta_i}\left[\frac{2\delta\Delta^T\nabla F(\theta)}{2}\right] + O(\delta^2) \nonumber\\
    \nonumber &= \frac{\Delta^T \nabla F(\theta)}{\Delta_i} + O(\delta^2).
\end{align}
Note here that the bias in the estimator above is $O(\delta^2)$, while the corresponding bias with two function measurements for (one-sided) GSPSA in \eqref{1spsa} is $O(\delta)$.

\subsection{Four-measurement B-GSPSA}
 Using the first two terms in \eqref{eq:eq_balanced}, we arrive at the four measurements B-GSPSA as follows:
\begin{align}
    \mathcal{D}_i^2F(\theta) &= \frac{1}{\delta \Delta_i} \left[\left(\frac{\tau_{\delta\Delta}-\tau_{-\delta \Delta}}{2}\right) - \left(\frac{(\tau_{\delta \Delta} - \tau_{-\delta \Delta})^3}{2.3.2^3}\right)\right]F(\theta) \nonumber\\ \nonumber
    &= \frac{1}{\delta\Delta_i} \left[\frac{27 \left( F(\theta+\delta\Delta)-F(\theta-\delta\Delta)\right)}{48}\right] \nonumber\\ \nonumber
    &\qquad- \frac{1}{\delta\Delta_i}\left[\frac{\left(F(\theta + 3\delta\Delta)-F(\theta -3\delta \Delta)\right)}{48}\right].
\end{align}
The RHS is obtained upon simplification. Using Taylor's expansions of $F(\theta\pm\delta\Delta)$ and $F(\theta \pm 3\delta\Delta)$ and simplifying, we obtain
\begin{align}
    \mathcal{D}_i^2 F(\theta) = \frac{\Delta^T \nabla F(\theta)}{\Delta_i} + O(\delta^4).\nonumber
\end{align}
Note that the zeroth order as well as the second, third and fourth order terms are exactly equal to zero here.

\begin{remark}
Recall that five-measurement (unbalanced) GSPSA that requires five function measurements provides a bias bound of $O(\delta^4)$, see \eqref{4spsa}. In comparison, four-measurement B-GSPSA corresponding here to $k_2=2$, i.e.,  the first two terms in the expansion (\ref{eq:eq_balanced}) results in a similar bias of $O(\delta^4)$.
\end{remark}

\subsection{$2 k_2$-measurement B-GSPSA}
Here the gradient estimator is obtained from \eqref{eq:eq_balanced}
 by truncating the series for a general $k_2 \geq 1$. 
We start with such an expression for $\mathcal{D}_i^{k_2}$, and simplify it further as follows: 
\begin{align}
&
    \delta \Delta_i \mathcal{D}_i^{k_2} \nonumber\\\nonumber
    & =
    \left (\frac{\tau_{\delta \Delta}-\tau_{-\delta \Delta}}{2}\right )\\  \nonumber&+ \sum_{j=1}^{k_2-1} \frac{(-1)^j (2j)!}{2^{2j}(j!)^2} \frac{1}{2j+1} \left ( \frac{\tau_{\delta \Delta} - \tau_{-\delta \Delta}}{2}\right)^{2j+1} \\\nonumber
    &=  \left ( \frac{\tau_{\delta \Delta} - \tau_{-\delta \Delta}}{2} \right ) \\\nonumber
    & + \sum_{j=1}^{k_2-1} \frac{(-1)^j (2j)!}{2^{4j+1}(j!)^2} \frac{1}{2j+1}\sum_{l=0}^{2j+1} (-1)^l{{2j+1} \choose l}(\tau_{\delta \Delta})^{2j+1-2l}\\\nonumber
    &= (\tau_{\delta \Delta} - \tau_{-\delta \Delta})\left[ \sum_{i = 0}^{k_2-1} \frac{(2i)!}{2^{4i+1}(i!)^2}\frac{1}{2i+1} {{2i+1}\choose {i}} \right]\\\nonumber
    &- (\tau_{3\delta \Delta} - \tau_{-3\delta \Delta}) \left [ \sum_{i = 1}^{k_2-1} \frac{(2i)!}{2^{4i+1}(i!)^2}\frac{1}{2i+1} {{2i+1}\choose {i-1}}\right ]\\\nonumber
    &+ (\tau_{5\delta \Delta}-\tau_{-5 \delta \Delta}) \left [ \sum_{i = 2}^{k_2-1} \frac{(2i)!}{2^{4i+1}(i!)^2}\frac{1}{2i+1}{{2i+1}\choose {i-2}}\right]\\\nonumber
    &- \ldots \\\nonumber
    &+ (-1)^{k_2-1}(\tau_{(2k_2-1)\delta \Delta} - \tau_{-(2k_2-1)\delta \Delta}) \\\nonumber
    &\qquad\times\left [ \sum_{i = k_2-1}^{k_2-1} \frac{(2i)!}{2^{4i+1}(i!)^2}\frac{1}{2i+1} {{2i+1}\choose {i-k_2+1}}\right ]\\\nonumber
    &= \sum_{j=0}^{k_2-1} (-1)^j (\tau_{(2j+1)\delta\Delta} - \tau_{-(2j+1)\delta \Delta})\\\nonumber
    &\qquad\quad\times \left [ \sum_{i=j}^{k_2-1} \frac{(2i)!}{2^{4i+1}(i!)^2}\frac{1}{2i+1}{{2i+1}\choose {i-j}}\right].\nonumber
\end{align}
 Thus, we have 
\begin{align}
\label{eq:bgpsa_noiseless_est}
    &
    \delta \Delta \mathcal{D}_i^{k_2} F(\theta) \nonumber\\
    &= \left[\sum_{j=0}^{k_2-1} \frac{\tau_{(2j+1)\delta \Delta} - \tau_{-(2j+1)\delta\Delta}}{2}\sum_{i=j}^{k_2-1} {\K}_i {{2i+1} \choose {i-j}}\right] F(\theta)\nonumber \\
    & = \sum_{j = 0}^{k_2-1} \frac{\left(F(\theta + (2j+1) \delta\Delta) - F(\theta - (2j+1)\delta\Delta)\right)}{2} \nonumber\\
    & \qquad \times \sum_{i = j}^{k_2-1} {\K}_i {2i+1 \choose i-j}, 
\end{align}
where ${\K}_i = \frac{(2i)!}{2^{4i}(i!)^2}\frac{1}{2i+1}$. 

In the general case with noisy observations, the $2k_2$-measurement B-GSPSA estimator is the following:
\begin{align}
    &
    \widetilde{\mathcal{D}}^{k_2}_i F(\theta(n)) \nonumber\\
    &\triangleq
      \frac{1}{\delta(n) \Delta_i(n)}\sum_{j=0}^{k_2 - 1}\frac{ y_n^+(j) - y_n^-(j)}{2}\sum_{i=j}^{k_2 - 1} {\K}_i {{2i+1} \choose {i - j}}, \label{eq:bgspsa-est}
\end{align}
where $y_n^+(j) = f(\theta(n) + (2j+1)\delta(n) \Delta(n), \xi_l^+(n)) $, and 
$y_n^-(j) = f(\theta(n) - (2j+1)\delta(n) \Delta(n), \xi_l^-(n))$, respectively. Further, $\{\xi_l^{\pm}(n)\}$ are independent sequences of i.i.d. noise random variables.  The overall algorithm flow is similar to the pseudo-code presented in Algorithm \ref{alg:GSPSA}, except that the gradient estimator used is \eqref{eq:bgspsa-est} and the update iteration is given by
\begin{align}
	&\theta_i(n+1) = \theta_i(n) - a(n) \widetilde{\cal D}^{k_2}_iF(\theta(n)).
 \label{eq:bgspsa-gd-update}
\end{align}

\subsection{Bias of B-GSPSA gradient estimator}
Since we go up to the $2k_2$-th term in Taylor series in arriving at \eqref{eq:bgpsa_noiseless_est}, we require the following variant of Assumption ~\ref{ass:derivative-bound}:

\vspace*{4pt}
\begin{assumption}
    \label{ass:derivative-bound_balanced}
    $F: \mathbb{R}^d \rightarrow \mathbb{R}$ is $2k_2$-times continuously differentiable with a bounded $(2k_2+1)$-th derivative.
\end{assumption}

\vspace*{4pt}
The result below provides an $O(\delta(n)^{2k_2})$ bias bound for B-GSPSA, which uses $2k_2$ function measurements. To obtain a matching bound with one-sided GSPSA, one would need $(2k_2+1)$ function measurements. This is an advantage with B-GSPSA estimators that they give the same order of accuracy as GSPSA estimators while requiring less number of function measurements. 

\vspace*{4pt}
 \begin{lemma}[Bias lemma]
\label{lemma:gspsa-bias-balanced}
\ \\Under \ref{ass: del iid} -- \ref{ass:f bound} and \ref{ass:derivative-bound_balanced}, for $\widetilde{\cal D}^{k_2}_iF(\theta(n)) $ defined according to \eqref{eq:bgspsa-est}
we have a.s. that $\text{ for } i=1,\ldots,d,$
\begin{align}
& \left| \E\left[\left.\widetilde{\cal D}^{k_2}_i F(\theta(n)) \right| \mathcal{F}_n \right] - \nabla_i F(\theta(n))\right| \le c_3\delta(n)^{2k_2},   \textrm{ and}\nonumber\\
& \E \left [\left \|\widetilde{\mathcal{D}}^{k_2}F(\theta(n)) - \E \left[ \widetilde{\mathcal{D}}^{k_2} F(\theta(n))\right]\right \|^2\right ] \leq \frac{c_4}{\delta(n)^2}, \nonumber
\end{align} 
where $c_3$ and $c_4$ are dimension-dependent constants.
\end{lemma}
\begin{proof}
    This proof can be derived in a similar manner as the proof of Lemma \ref{lemma:gspsa-bias}. 
\end{proof}

\section{Main Results}
\label{convergence}

In this section, we provide both asymptotic convergence results as well as non-asymptotic bounds for Algorithm \ref{alg:GSPSA}, which uses the
$(k_1+1)$-measurement GSPSA gradient estimator discussed in Section \ref{sec:4spsa} for a general $1 \leq k_1 < \infty$. By using completely parallel arguments, one can obtain similar results for B-GSPSA, and we omit the details.

For the sake of asymptotic analysis, we make the following assumptions in addition to \ref{ass:derivative-bound} to \ref{ass:f bound} specified earlier:

\begin{assumption}
    \label{ass:stability condition}
     $\sup_{n}\|\theta(n) \| < \infty$ w.p. $1$. 
\end{assumption}
\begin{assumption}
\label{ass:step_size}
    The step-sizes $a(n)$ and perturbation parameters $\delta(n)$ are positive, for all $n$ and satisfy
\begin{align*}
&a(n), \delta(n) \rightarrow 0\text{ as } n \rightarrow \infty, 
\sum_n a(n)=\infty,\\
&\text{ and } \sum_n \left(\frac{a(n)}{\delta(n)}\right)^2 <\infty.
\end{align*}
\end{assumption}
The above assumptions are commonly used for  the analysis of simultaneous perturbation-based stochastic gradient algorithms, cf. \cite{spall1992multivariate,prashanth2017rdsa,bhatnagar-book}. 
\begin{theorem}[Strong convergence]
    \label{thm:asymp-conv}
    Assume \ref{ass:derivative-bound} -- \ref{ass:f bound} and \ref{ass:stability condition} -- \ref{ass:step_size}.	
    Let $\bar{H}$ denote the largest internally chain recurrent set contained in $ C \triangleq \{ \theta \mid \nabla f(\theta) = 0 \}$.
    Then, the iterates $\theta(n), n \geq 0$, updated according to either \eqref{eq:gspsa-gd-update} or \eqref{eq:bgspsa-gd-update}, satisfy 
    \[\theta(n) \rightarrow \bar H \text{ a.s. as } n\rightarrow \infty.\]
\end{theorem}
\begin{proof}
    See Section \ref{sec:proof-asymp}.
\end{proof}

For the non-asymptotic analysis, we require the following additional assumption:
\vspace*{4pt}
\begin{assumption}
\label{ass: gradeint bound}
    There exists a constant $ B > 0 $ such that $\| \nabla f ( x ) \|_1 \leq B, \forall x \in \R^d$.
\end{assumption}
\vspace*{4pt}

\begin{theorem} \label{thm:nonasymp} Suppose the objective function $F$ is $L$-smooth\footnote{A function $F$ is $L$-smooth if it is Lipschitz continuous with Lipschitz constant $L>0$, i.e., it satisfies, 
		$	\| \nabla F ( x ) - \nabla F ( y ) \|  \leq L \| x - y \|,
		\quad \forall x , y \in \mathbb { R } ^ {d }.$}, and assumptions \ref{ass:derivative-bound} -- \ref{ass:f bound} and \ref{ass:stability condition} -- \ref{ass: gradeint bound} hold. Suppose that the  algorithm \eqref{eq:gspsa-gd-update} is run with the stepsize $ a(n)=a $ and perturbation constant $  \delta(n)=\delta $ for each $n=1,\ldots,m$, where
	\begin{align} 
		a =  \min \bigg\{\frac{1}{L}, \frac{1}{m^{\frac{k_1+2}{2k_1+2}}}\bigg\}, \text{ }  \delta = \frac{1}{m^{1/(2k_1+2)}}. \label{eq:biased_sp_par}
	\end{align}
	Let $\theta(R)$ be picked uniformly at random from the set $\{\theta(1), \ldots, \theta(m)\}$. Then, for any $ m \ge 1 $, we have
	\begin{align} 
		& \mathbb { E }  \left\| \nabla F \left( \theta(R) \right) \right\| ^ { 2 }   \le \frac{ 2 L (F(\theta_1) - F(\theta^*)) }{{ m }} + \frac{\mathcal{K}_1}{m^{\frac{k_1}{2k_1+2}}}, \label{eq:zrsg_o2}
	\end{align}
	where $ \theta^* $ is a global optimum of $F$ and $\mathcal{K}_1=4 B c_1+  \frac{L d c_1^2 }{ m} + {L  c_2}$.
\end{theorem}
\begin{proof}
    See Section \ref{sec:proof-nonasymp}.
\end{proof}
\vspace*{4pt}
\begin{remark}
The non-asymptotic bound above implies that $O\left(\frac{1}{\epsilon^{2+\frac{2}{k_1}}}\right)$ number of iterations are sufficient to find an $\epsilon$-stationary point, i.e., $\mathbb { E }  \left\| \nabla F \left( \theta(R) \right) \right\| ^ { 2 } \le \epsilon$. For $k_1 >2$, this iteration complexity is better than that of regular SPSA, which requires  $O\left(\frac{1}{\epsilon^{3}}\right)$ number of iterations to find the same point.  Moreover, choosing larger values of $k_1$ would result in an iteration complexity that is nearly $O\left(\frac{1}{\epsilon^{2}}\right)$ -- a bound that one would obtain for a model with unbiased gradient information.
\end{remark}

\section{Generalization to other simultaneous perturbation-based gradient estimators}
\label{sec:unified_estimator}

In this section, we first present GSPGS estimators in their full generality using random perturbation sequences with various perturbation distributions. 
In the case when the parameter $k_1=1$, we recover popular simultaneous perturbation-based gradient estimators such as the smoothed functional, RDSA, and SPSA, from the generalized estimator presented below.

\subsection{Generalized gradient estimators: One-sided case}
\label{ub1}

In \eqref{eq:gspsa_noiseless_est}, using a Bernoulli random perturbation vector $\Delta$, we arrived at the $(k_1 + 1)$-measurement GSPSA estimate in the noiseless setting. Note, however, that Bernoulli distribution is only one of the choices for the random perturbations that can be used in the gradient estimator. In this section, we provide a generalized gradient estimator, which includes several other distributions as options for setting the random perturbations.

As a gentle start, in the noise-less case, with random vectors $U(n)$ and $V(n)$, $n\geq 0$, satisfying certain distribution assumptions (to be specified later), 
the GSPGS estimator is formed as follows:
\begin{align}
	&
	{\cal D}^{k_1}F(\theta(n))= \nonumber\\ 
	&\frac{1}{\delta(n)} \sum_{l=0}^{k_1} V(n)\frac{(-1)^{1-l} C^{k_1}_l  F(\theta(n)+l\delta(n)U(n))}{l!}.\label{eq:gspg}
\end{align}

In the general case with noisy observations, the GSPGS estimator is formed as follows: 
\begin{align}
	&
	\widehat{\cal D}^{k_1}F(\theta(n))= \nonumber\\ 
	&\frac{1}{\delta(n)} \sum_{l=0}^{k_1} V(n)\frac{(-1)^{1-l} C^{k_1}_l  f(\theta(n)+l\delta(n)U(n), \xi_l(n))}{l!},\label{eq:gspg-noisy}
\end{align}
where $\xi_l(n), n\geq 0,l = 0,1,\cdots,k_1$ are i.i.d random variables having a common distribution.
In the above, $U(n)$, $V(n)$ are chosen such that for all $n\geq 0$, \[\E\left[ V(n)U(n)\tr\mid \F_n\right]=I \mbox{ and }
\E[V(n)|\F_n] = 0,\] 
where 
$\F_n=\sigma(\theta(j), j\leq n, V(j),\xi_0(j),\ldots, \xi_{k_1}(j),j < n), n\geq 1$. We do not include the random variables $U(j),j\geq 0$, in the sigma algebra since in all the cases (as shown below), $V(n)$ is a scalar multiple of $U(n)$.
\vspace*{4pt}
\subsubsection{Special cases}
We now discuss a few choices for $U,V$ in the GSPGS estimator \eqref{eq:gspg-noisy} defined above.
\begin{itemize}
	\item Let $U \sim \N (0, I_d)$, where $\N (0, I_d)$ represents the $d$-dimensional Gaussian distribution and let $V = U$. In this setting, the GSPGS estimator turns out to be a generalization of the smoothed functional (SF) estimator proposed in \cite{katkul}. In particular, for $k_1=1$, the GSPGS estimator coincides with the well-known one-sided Gaussian SF scheme for gradient estimation.
	
	\item Let $U$ be uniformly distributed over a surface of the $d$-dimensional unit sphere and let $V = dU$. In this case, GSPGS is a generalization of the random direction stochastic approximation (RDSA) scheme proposed in \cite{kushcla}. As before, $k_1=1$ leads to the vanilla one-sided RDSA estimate. 
 
 \item A later refinement of RDSA, see \cite{prashanth2017rdsa}, with $k_1=1$  corresponds to $U$ being uniformly distributed over $[-\eta, \eta]$ and $V = \frac{3}{\eta^2}U$. An alternate choice, proposed in the aforementioned reference, is with asymmetric Bernoulli perturbations, i.e., $U$ takes values $-1$ and $1+\epsilon$ with probabilities $\frac{1+\epsilon}{2+\epsilon}$ and $\frac{1}{1+\epsilon}$ respectively and $V = \frac{1}{1+\epsilon}U$. The GSPGS estimator with either choice for $U,V$ would serve as a generalization of the RDSA schemes proposed in \cite{prashanth2017rdsa}.
	
	\item Let $U$ to be set as symmetric $\pm 1$-valued Bernoulli random variable and $V = U$. In this case, we obtain the GSPSA estimator \eqref{eq:gspsa-est}, with $k_1=1$ turning out to be the well-known SPSA scheme from \cite{spall1992multivariate}.
\end{itemize}

\subsubsection{Bias in one-sided GSPGS estimator}
For bounding the bias of the estimator \eqref{eq:gspg-noisy}, we make the following assumptions:
\vspace*{4pt}
\begin{assumption}
    \label{ass:derivative_uni}
    $F: \mathbb{R}^d \rightarrow \mathbb{R}$ is $k_1$-times continuously differentiable with a bounded $(k_1+1)$-th derivative.
\end{assumption}
\vspace*{4pt}
\begin{assumption}
    \label{ass:expectation_uni}
    For all $n$,
    $\E\left [V(n)U(n)^T\right |\mathcal{F}_n] = I$, $\E[V(n)|\mathcal{F}_n] = 0$ and $\E\left[V(n)^2|\mathcal{F}_n\right] \leq \sigma^2$. 
    
\end{assumption}
\vspace*{4pt}
\begin{assumption}
    \label{ass:noise_uni}
    The noise random variables $\{\xi_l(n)\}$, for all $l = 0,\ldots k_1$ satisfy 
    $\E\left [ \xi_l(n)| \mathcal{F}_n\right] = 0$.
    
\end{assumption}
\vspace*{4pt}
\begin{assumption}
    \label{ass:stability_uni}
    $\sup_{n} \E\left[f(\theta(n) + l\delta(n)U(n), \xi_l(n))^2 | \F_n\right] \leq M < \infty$ for all $l = 1, \ldots, k_1$.
\end{assumption}
\vspace*{4pt}
Note here that \ref{ass:derivative_uni} is the same as \ref{ass:derivative-bound}. The main result that bounds the bias of GSPGS estimator is given below.
\begin{lemma}
\label{lemma:bias_uni}
    Under assumptions \ref{ass:derivative_uni} - \ref{ass:stability_uni}, the unified GSPSA gradient estimator \eqref{eq:gspg-noisy} satisfies the following bounds:
    \begin{align*}
        &\left\| \E \left[ \widehat{\cal{D}}^{k_1}F(\theta(n))\right|\mathcal{F}_n] - \nabla F(\theta(n))\right \| \leq C_1 \delta(n)^{k_1}, \text{ and}\\
        &\E \left [\left \|\widehat{\mathcal{D}}^{k_1}F(\theta(n)) - \E \left[ \widehat{\mathcal{D}}^{k_1} F(\theta(n))\right]\right \|^2\right ] \leq \frac{C_2}{\delta(n)^2},
    \end{align*}
    where $C_1$ and $C_2$ are dimension-dependent constants. 
\end{lemma}
\begin{proof}
    See Section \ref{sec:gspsa_bias_uni}.
\end{proof}
\subsection{Generalized gradient estimators: Balanced case}
\label{b1}

 In this section, we provide a balanced or two-sided version of the generalized gradient estimator \eqref{eq:gspg-noisy}. As in the previous section, this generalization allows several other distributions as options for setting the random perturbations. 

In the noiseless setting, the balanced version of generalized simultaneous perturbation-based gradient search (B-GSPGS) estimator with random vectors $U$ and $V$ is formed as follows:
\begin{align}
\label{eq:gspsa-est2_balanced_noiseless}
    &
    \mathcal{D}^{k_2} F(\theta(n)) \!=\!
      \frac{1}{\delta(n)}\sum_{j=0}^{k_2} V(n)\frac{ y_n^+(j) - y_n^-(j)}{2}\sum_{i=j}^{k_2} {\K}_i {{2i+1} \choose {i - j}}, 
\end{align}
where $y_n^+(j) = f(\theta(n) + (2j+1)\delta(n) U(n)) $ and $y_n^-(j) = f(\theta(n) - (2j+1)\delta(n) U(n))$.

In the general case with noisy observations, the B-GSPGS estimator is defined as follows: 
\begin{align}
\label{eq:gspsa-est2_balanced}
    &
    \widetilde{\mathcal{D}}^{k_2} F(\theta(n)) =
      \frac{1}{\delta(n)}\sum_{j=0}^{k_2} V(n)\frac{ y_n^+(j) - y_n^-(j)}{2}\sum_{i=j}^{k_2} {\K}_i {{2i+1} \choose {i - j}}, 
\end{align}
where now $y_n^+(j) = f(\theta(n) + (2j+1)\delta(n) U(n), \xi_j^+(n)) $ and $y_n^-(j) = f(\theta(n) - (2j+1)\delta(n) U(n), \xi_j^-(n))$. As before, $\xi_j^+(n), \xi_j^-(n), n\geq 0, j = 0,1,\cdots, k_2$ are assumed to be i.i.d random variables. 

Since we go up to the $2k_2$-th term in Taylor series in arriving at \eqref{eq:gspsa-est2_balanced}, we require the following variant of \ref{ass:derivative_uni} (that is the same as \ref{ass:derivative-bound_balanced}):
\vspace*{4pt}
\begin{assumption}
    \label{ass:derivative_uni_balanced}
    $F: \mathbb{R}^N \rightarrow \mathbb{R}$ is $2k_2$-times continuously differentiable with a bounded $(2k_2+1)$-th derivative.
\end{assumption}
\vspace*{4pt}
\begin{lemma}
    \label{lemma:bias_uni_balanced}
    Under assumptions \ref{ass:expectation_uni} - \ref{ass:stability_uni} and \ref{ass:derivative_uni_balanced}, the unified B-GSPGS estimator \eqref{eq:gspsa-est2_balanced} satisfies the following bounds:
    \begin{align*}
        &\left\| \E \left[ \widetilde{\cal{D}}^{k_2}F(\theta(n))\right|\mathcal{F}_n] - \nabla F(\theta(n))\right \| \leq C_3 \delta(n)^{2k_2}, \text{ and}\\
        &\E \left [\left \|\widetilde{\mathcal{D}}^{k_2}F(\theta(n)) - \E \left[ \widetilde{\mathcal{D}}^{k_2} F(\theta(n))\right]\right \|^2\right ] \leq \frac{C_4}{\delta(n)^2},
    \end{align*}
    where $C_3$ and $C_4$ are dimension-dependent constants. 
\end{lemma}

\begin{proof}
See Section \ref{sec:gspsa_bias_uni_balanced}.
\end{proof}
\begin{remark}
As with B-GSPSA estimators presented in Section~\ref{ub1}, we can obtain through suitable choice of $U(n)$ and $V(n)$ in the B-GSPGS estimator (\ref{eq:gspsa-est2_balanced}), as explained in the above section, the balanced versions of the random perturbation estimators GSF and GRDSA as well. 
\end{remark}

\section{Convergence Proofs}
\label{proofs}

\subsection{Proof of Lemma \ref{lemma:bias_uni}}
\label{sec:gspsa_bias_uni}
We first state and prove a result below for the noiseless version of the GSPGS estimator, defined in \eqref{eq:gspg}. Subsequently, we use this result to prove the claim in Lemma \ref{lemma:bias_uni}.
For ease of notation, we drop below the dependence of quantities such as $\theta$, $V$, and $U$ on the iterate index $n$.

\begin{lemma}
	\label{lemma:uni_gspsa}
	Assume \ref{ass:derivative_uni}. Then, for any $k_1 \geq 1$,  we have
	\begin{equation} 
		\label{kspsa}
		{\cal D}^{k_1}F(\theta) = VU^T \nabla F(\theta) + O(\delta^{k_1}).
	\end{equation}
\end{lemma}
\begin{proof}
We will use the following identities in the proof:
    \begin{align}
    &\frac{1}{1} {k \choose 1} - \frac{1}{2} {k \choose 2} + \cdots + (-1)^{k + 1} \frac{1}{k} {k \choose k} = \sum_{j = 1}^k \frac{1}{j},\label{eq:ident1}\\
    &{k \choose 1} - {k \choose 2} + {k \choose 3} - \ldots (-1) ^{k+1} {k \choose k} = 1,\label{eq:ident2}    \\
 &\sum_{j = 0} ^k (-1)^{k-j} {k \choose j} j^q = 0 \textrm{ for any }0 < q < k.\label{eq:ident3}
    \end{align}
Notice that
\begin{align*}
    &{\cal D}^{k_1} F(\theta) \\
    &= \frac{1}{\delta} \sum_{l=0}^{k_1} V\frac{(-1)^{1-l} C^{k_1}_l F(\theta+l\delta U)}{l!}\\
     & = \frac{1}{\delta}\left [\left (-\sum_{j = 1}^{k_1} \frac{1}{j} \right)V F(\theta) + \sum_{l=1}^{k_1}V \frac{(-1)^{1-l} C^{k_1}_l F(\theta+l\delta U)}{l!}\right ] \\
    & = \frac{1}{\delta}\left [\left (-\sum_{j = 1}^{k_1} \frac{1}{j} \right) V F(\theta) + \sum_{l=1}^{k_1} V \frac{(-1)^{1-l} {k_1 \choose l} F(\theta+l\delta U)}{l}\right].
\end{align*}

Using Taylor's expansion for $v \geq 1$, we obtain
\begin{align*}
F(\theta+v\delta U) &= F(\theta) +v\delta U^T \nabla F(\theta)+\frac{(v\delta)^2 U^T\nabla^2F(\theta) U}{2!}\\
&\qquad+ \ldots +\frac{(v\delta)^{k_1} \nabla^{k_1} F(\theta)(U\otimes U \otimes \ldots \otimes U)}{k_1!}\\
& \qquad+ O(\delta^{k_1+1}).
\end{align*}
In the following, we calculate the coefficients of $VF(\theta)$, $VU^T\nabla F(\theta)$, etc. 

Using \eqref{eq:ident1}, the coefficient of $F(\theta)$ is simplified as follows:
\begin{align*}
    & - \sum_{j = 1}^{k_1} \frac{1}{j} + \frac{1}{1} {k_1 \choose 1} - \frac{1}{2} {k_1 \choose 2} + \cdots + (-1)^{k_1 + 1} \frac{1}{k_1} {k_1 \choose k_1}\\
    & =- \sum_{j = 1}^{k_1} \frac{1}{j} + \sum_{j = 1}^{k_1} \frac{1}{j}= 0.
\end{align*}

 Using \eqref{eq:ident2}, the coefficient of $VU^T\nabla F(\theta)$ turns out to be one from the following calculation:
 \begin{align*}
     & \frac{1}{1}{k_1 \choose 1} \frac{1}{1!} - \frac{1}{2} {k_1 \choose 2} \frac{2}{1!} + \ldots + (-1)^{k_1+1} \frac{1}{k_1} {k_1 \choose k_1} \frac{k_1}{1!}\\
     &= {k_1 \choose 1} - {k_1 \choose 2}  + \ldots + (-1)^{k_1+1} {k_1 \choose k_1} = 1.
 \end{align*}

 Using \eqref{eq:ident3}, the coefficient of $V\nabla^q F(\theta)(U\otimes U \otimes\ldots\otimes U)$, for any $2 \leq q \leq k_1$,  can be simplified as follows:
 \begin{align*}
     & \frac{1}{1} {k_1 \choose 1} \frac{1^q}{q!} - \frac{1}{2} {k_1 \choose 2} \frac{2^q}{q!} + \ldots + (-1)^{k_1+1} \frac{1}{k_1} {k_1\choose k_1} \frac{k_1^q}{q!}\\
     &= \frac{1}{q!}\left [ \sum_{j = 0}^{k_1} (-1)^{j+1} {k_1 \choose j} j^{q-1} \right ] = 0.
 \end{align*}
Thus, from the foregoing, we have
 \begin{align*}
     {\cal D}^{k_1}F(\theta) &= \frac{1}{\delta} \left [{\delta VU^T \nabla F(\theta)} + O(\delta^{k_1+1})\right ]\\
     &= VU^T \nabla F(\theta) + O(\delta^{k_1}).
 \end{align*}
 \end{proof}
 \begin{proof}~\textbf{\textit{(Lemma \ref{lemma:bias_uni})}}\\
The proof technique used here is similar to that of the well-known simultaneous perturbation-based gradient estimator, cf. \cite[Lemma 1]{spall1992multivariate,prashanth2017rdsa} or \cite[Chapter 5]{bhatnagar-book}.

Using $F(\theta) = E_\xi[f(\theta,\xi)]$, we have \\
\[
\E\left[\widehat{\cal D}^{k_1}_iF(\theta(n))\mid \mathcal{F}_n\right] = \E\left[{\cal D}^{k_1}_iF(\theta(n))\mid \mathcal{F}_n\right]. 
\]
From the analysis using Taylor's expansions in Section \ref{sec:kspsa} and Lemma \ref{lemma:uni_gspsa}, we have
\begin{align}
\nonumber
 \E\left[{\cal D}^{k_1}F(\theta(n))\mid \mathcal{F}_n\right] &= \E\left[V(n)U(n)^T \nabla F(\theta(n))\mid \mathcal{F}_n\right]\\
 \label{ic1}
& \qquad \quad + O(\delta(n)^{k_1}).
\end{align}
The first term on the RHS above arises also in Taylor's expansions using regular SPSA estimators. From \ref{ass:expectation_uni}, we have
\begin{align*}
\E\left[V(n)U(n)^T \nabla F(\theta(n))\mid \mathcal{F}_n\right] &= \nabla F(\theta(n)).
\end{align*}
The first claim follows from (\ref{ic1}).
For the second claim, notice that
\begin{align}
    &\E{\left\| \widehat{\cal D}^{k_1} F(\theta(n)) -  \E\left[\widehat{\cal D}^{k_1} F(\theta(n))\right] \right\|^2} \le  \E{\left\| \widehat{\cal D}^{k_1} F(\theta(n))  \right\|^2} \nonumber\\
    & \le  \frac{1}{\delta(n)^2} \sum_{l=0}^{k_1} \E\left[V(n) f(\theta(n)+l\delta(n)U(n), \xi_l(n))\right]^2 \leq \frac{C_2}{\delta(n)^2}, \nonumber
\end{align}
where $C_2 = k_1\sigma^2 M$. 
\end{proof}
\subsection{Proof of Lemma \ref{lemma:bias_uni_balanced}}
\label{sec:gspsa_bias_uni_balanced}
First we will prove the following lemma (again after dropping the iteration index $n$ in the various quantities):
\begin{lemma}
	\label{lemma:uni_gspsa_balanced}
	 Let the assumption \ref{ass:derivative_uni_balanced} hold. Then for any $k_2\geq 1$,
	\begin{equation}
		\mathcal{D}^{k_2} F(\theta) = VU^T \nabla F(\theta) + O(\delta^{2k_2}).
	\end{equation}
\end{lemma}
\begin{proof}
We will use the following identities in the proof:
\begin{align}
&\sum_{i = 0}^{k} (-1)^i {{2k+1}\choose {k-i}} (2i+1) = 0,\label{eq:identity4}\\ 
&\sum_{j = 0}^{k} (-1)^j \sum_{i = j}^{k} \frac{(2i)!}{2^{4i}(i!)^2}\frac{1}{2i+1}{2i+1 \choose i-j} (2j+1)^q  = 0,\label{eq:identity5}
\end{align}
where $q$ is an odd integer and $1< q \leq 2k+1$.

We have
\begin{align*}
    &
    \delta \mathcal{D}_i^{k_2} F(\theta) \\
    & = \sum_{j = 0}^{k_2-1} V \frac{F(\theta + (2j+1) \delta U) - F(\theta - (2j+1)\delta U)}{2}\\
    &\qquad\quad\times\sum_{i = j}^{k_2-1}{\K}_i {2i+1 \choose i-j}. \nonumber
\end{align*}
Using Taylor series expansions, all even-order derivatives got cancelled. We only need to calculate the coefficient of odd-order derivatives.
Now we calculate the coefficient of $VU^T\nabla F(\theta)$, $V\nabla^3 F(\theta)(U\otimes U \otimes U)$ and so on.

 Using \eqref{eq:identity4}, the coefficient of $VU^T\nabla F(\theta)$ can be derived as follows: 
\begin{align}
    &  1+ \sum_{j=1}^{k_2-1}{\mathcal{K}}_j \sum_{i=0}^{j} (-1)^i {{2j+1}\choose {j-i}}(2i+1)\nonumber \\\nonumber
    & = 1 + 0 = 1.\nonumber 
\end{align}

 Using \eqref{eq:identity5}, the coefficient of $V\nabla^q F(\theta)(U\otimes U \otimes \ldots \otimes U)$, for any odd integer $q$ satisfying  $1 < q \leq 2k_2-1$, can be calculated as follows: 
\begin{align}
    &\sum_{j=0}^{k_2-1} (-1)^{j} \sum_{i=j}^{k_2-1} {\mathcal{K}}_i{{2i+1}\choose {i-j}} (2j+1)^q \nonumber\\
    &=\sum_{j=0}^{k_2-1} (-1)^{j} \sum_{i=j}^{k_2-1} \frac{(2i)!}{2^{4i}(i!)^2}\frac{1}{2i+1}{{2i+1}\choose {i-j}} (2j+1)^q = 0.\nonumber
\end{align}
From the foregoing, 
 \begin{align*}
     {\cal D}^{k_2}F(\theta) &= \frac{1}{\delta} \left [{\delta VU^T \nabla F(\theta)} + O(\delta^{2k_2+1})\right ]\\
     &= VU^T \nabla F(\theta) + O(\delta^{2k_2}).
 \end{align*}
\end{proof}
\begin{proof}~\textbf{\textit{(Lemma \ref{lemma:bias_uni_balanced})}}\\
The proof technique used here is similar to Lemma \ref{lemma:bias_uni}. 

Using $F(\theta) = E_\xi[f(\theta,\xi)]$, we have \\
\[
\E\left[\widetilde{\cal D}^{k_2}_iF(\theta(n))\mid \mathcal{F}_n\right] = \E\left[{\cal D}^{k_2}_iF(\theta(n))\mid \mathcal{F}_n\right]. 
\]

From the analysis using Taylor's series expansions in Section \ref{sec: balanced_estimator12} and Lemma \ref{lemma:uni_gspsa_balanced}, we have
\begin{align*}
 \E\left[{\cal D}^{k_2}F(\theta(n))\mid \mathcal{F}_n\right] &= \E\left[V(n)U(n)^T \nabla F(\theta(n))\mid \mathcal{F}_n\right]\\
& \qquad \quad + O(\delta(n)^{2k_2}).
\end{align*}
Using \ref{ass:expectation_uni}, the first term on the RHS above can be derived as 
\begin{align*}
\E\left[V(n)U(n)^T \nabla F(\theta(n))\mid \mathcal{F}_n)\right] &= \nabla F(\theta(n)).
\end{align*}
The first claim follows. 
For the second claim, notice that
\begin{align}
    &\E{\left\| \widetilde{\cal D}^{k_2} F(\theta(n)) -  \E\left[\widetilde{\cal D}^{k_2} F(\theta(n))\right] \right\|^2} \le  \E{\left\| \widetilde{\cal D}^{k_2} F(\theta(n))  \right\|^2} \nonumber\\
    &\leq \frac{C_4}{\delta(n)^2}, \nonumber
\end{align}
where $C_4$ is a constant. 
\end{proof}

\subsection{Proof of Lemma \ref{lemma:gspsa-bias}}\label{sec:proof_gspsa}
\begin{proof}
This proof follows from the proof of  Lemma \ref{lemma:bias_uni} by setting  $U$ to a $d$-vector of symmetric $\pm 1$-valued Bernoulli random variables and $V = U$. 
\end{proof}

\subsection{Proof of Theorem \ref{thm:asymp-conv}}
\label{sec:proof-asymp}
\begin{proof}
	We provide the proof for the GSPSA algorithm, with update iteration \eqref{eq:gspsa-gd-update}. The proof for B-GSPSA algorithm follows by a completely parallel argument, with Lemma \ref{lemma:gspsa-bias-balanced} in place of Lemma \ref{lemma:gspsa-bias}.
	
We first rewrite \eqref{eq:gspsa-gd-update} as follows:
\begin{equation}
\label{RM1}
\theta(n+1) = \theta(n) -a(n) (\nabla F(\theta(n)) + \eta_n + \beta_n),
\end{equation}
where $\eta_n = \widehat{\mathcal{D}}^{k_1}F(\theta(n)) - \E[\widehat{\mathcal{D}}^{k_1} F(\theta(n))| \F_n]$ and 
$\beta_n = \E[\widehat{\mathcal{D}}^{k_1} F(\theta(n))|\F_n] - \nabla F(\theta(n))$
are the martingale difference and bias terms, respectively.

	The proof of asymptotic convergence of \eqref{eq:gspsa-gd-update} requires invocation of the well-known Kushner-Clark Lemma (see Theorem 2.3.1 on pp. 29 of \cite{kushcla}). For this invocation, we need to verify that the assumptions A2.2.1 to A2.2.3 and A2.2.4'' of \cite{kushcla} are satisfied. Note that, 
 \begin{enumerate}
     \item $f \in {\cal {C}}^{k_1}$ implies A2.2.1.

     \item We have $\beta_n = O(\delta(n)^{k_1})$. Then by Lemma \ref{lemma:gspsa-bias} and Assumption \ref{ass:step_size}, $\beta_n \rightarrow 0$ (A2.2.2). 
     \item \ref{ass:step_size} clearly implies A2.2.3. 
     \item By using Doob martingale inequality, we get
     \begin{align}
         \ \mathbb{P}(\sup_{m \geq n} \left\| \sum_{i=n}^m a(i)\eta_i\right \| \geq \epsilon) &\leq \frac{1}{\epsilon^2} \mathbb{E} \left \| \sum_{i=n}^\infty a(i)\eta_i\right \|^2\nonumber\\
     &\leq \frac{1}{\epsilon^2}\sum_{i=n}^\infty
     a(i)^2\mathbb{E} \left \| \eta_i \right \|^2.\nonumber
     \end{align}

     Using \ref{ass:step_size} and $\mathbb{E} \left \| \eta_n\right \|^2 \leq \frac{c_2}{\delta(n)^2}$, which was shown in Lemma \ref{lemma:gspsa-bias}, we obtain
   \begin{align}
       \lim_{n \rightarrow \infty} \mathbb{P}(sup_{m\geq n} \left \|\sum_{i=n}^m a(i)\eta_i\right \| \geq \epsilon) & \leq \frac{c_2}{\epsilon^2}\lim_{n \rightarrow \infty} \sum_{i=n}^\infty \frac{a(i)^2}{\delta(i)^2}\nonumber\\
       &= 0, \text{by}\ref{ass:step_size}. \nonumber
   \end{align}
 \end{enumerate}
	The main claim now follows by invoking the Kushner-Clark lemma.
\end{proof}

\begin{table*}[!hp]
\centering
 \caption{Parameter error \eqref{eq:param_error} for the GSPSA algorithm with noise parameter $\sigma=0.001$ under two different objective functions. The results are averages over $20$ independent replications. GSPSA requires $(k_1+1)$ function measurements where $k_1=1$ corresponds to the well-known one-sided SPSA algorithm. Here $k_1$ denotes the truncation parameter with the form of the gradient estimator used being (\ref{eq:gspsa-est}).}
\label{tab:GSPSA}
\begin{tabular}{|c|c|c|c|c|}
\toprule
\rowcolor{gray!20}
\multicolumn{5}{|c|}{\multirow{2}{*}{\textbf{GSPSA with Rastrigin objective function \eqref{eq:Rastrigin}}}}\\[1em]
\midrule
 \textbf{Truncation Parameter $(k_1)\rightarrow$} & \multirow{2}{*}{$k_1=1$} & \multirow{2}{*}{$k_1=2$} & \multirow{2}{*}{$k_1 =3$} & \multirow{2}{*}{$k_1 =4$}  \\
 \textbf{Dimension $(d)$}  & & & &\\
 \textbf{$\downarrow$}  & & & &\\
 \midrule
$d=5$ & $5.64E^{-2}$ & $5.3E^{-2} $ &$2.99E^{-2}$ & $\bm{1.39E^{-2}}$\\
&&&&\\
$d=10$ &$5.64E^{-2} $ & $5.3E^{-2} $ &$3E^{-2}$ &$\bm{1.46E^{-2}}$ \\ 
&&&&\\ 
$d=50$ &$5.64E^{-2} $ & $5.3E^{-2} $ &$3.26E^{-2}$ &$\bm{1.01E^{-2}}$ \\ 
&&&&\\
$d=100$& $ 5.69E^{-2}$& $5.29E^{-2}$ &$2.96E^{-2}$ &$\bm{9.81E^{-3}}$ \\

 \bottomrule
 \rowcolor{gray!20}
\multicolumn{5}{|c|}{\multirow{2}{*}{\textbf{GSPSA with the quadratic objective function \eqref{eq:quadratic}}}}\\[1em]
\midrule
 \textbf{Truncation Parameter  $(k_1)\rightarrow$} & \multirow{2}{*}{$k_1 = 1$} & \multirow{2}{*}{$k_1 = 2$} & \multirow{2}{*}{$k_1 = 3$} & \multirow{2}{*}{$k_1 = 4$}  \\
 \textbf{Dimension $(d)$}  & & & &\\
 \textbf{$\downarrow$}  & & & &\\
 \midrule
$d = 5$ & $7.88E^{-3} $ & $\bm{9.11E^{-4}} $ & $1.18E^{-3} $&$1.65E^{-3}$\\
&&&&\\
$d = 10$ &$ 4.16E^{-2}$ & $ 1.42E^{-2}$ &$ \bm{1.3E^{-2}}$ &$1.41E^{-2} $\\ 
&&&&\\
$d = 50$ &$ 1.7E^{-1}$ & $1.6E^{-1} $ & $1.6E^{-1} $&$\bm{1.5E^{-1}}$\\ 
&&&&\\
$d = 100$& $2.2E^{-1}$& $2.2E^{-1}$ &$2.1E^{-1} $ &$\bm{2.1E^{-1}} $\\
&&&&\\
 \bottomrule 
\end{tabular}
\end{table*}

\begin{table*}
\captionsetup[subtable]{position = below}
	\captionsetup[table]{position=top}
 \caption{Parameter error \eqref{eq:param_error} for GRDSA, GSF and B-GSPSA algorithms under two different objective functions with noise parameter $\sigma=0.001$. The results are averages over $20$ independent replications. B-GSPSA with $k_2=1$ corresponds to the well-known SPSA algorithm \cite{spall1992multivariate}.
 Here $k_1$ is the truncation parameter for GRDSA/GSF while $k_2$ is the truncation parameter for B-GSPSA. Recall that the number of function measurements is $k_1+1$ for the former and $2k_2$ for the latter.
 }
\label{tab:nonGSPSAresults}
 \begin{subtable}{0.333\textwidth}
		\centering
\caption{Parameter error for GRDSA}
\label{tab:GRDSA}	
\begin{tabular}{|c|c|c}
\toprule
\rowcolor{gray!20}
\multicolumn{3}{|c}{\multirow{2}{*}{\textbf{GRDSA with Rastrigin objective}}}\\[1em]
\midrule
 \textbf{Parameter$(k_1) \; \rightarrow$} & \multirow{2}{*}{$k_1=1$} & \multirow{2}{*}{$k_1=4$}   \\
 \textbf{Dimension $(d)$}  & & \\
 \textbf{$\downarrow$}  & & \\
 \midrule
$d=5$ & $3.45E^{-3}$ & $\bm{2.54E^{-4}} $  \\
&&\\
$d=10$ &$1.22E^{-2}$ & $\bm{1.07E^{-3}}$  \\ 
&&\\ 
$d=50$ &$2.1E^{-1} $ & $ \bm{6.87E^{-3}}$  \\ 
&&\\
$d=100$& $7.4E^{-1} $& $\bm{1.18E^{-2}}$  \\
 \bottomrule
 \rowcolor{gray!20}
\multicolumn{3}{|c}{\multirow{2}{*}{\textbf{GRDSA with quadratic objective}}}\\[1em]
\midrule
 \textbf{Parameter $(k_1) \; \rightarrow$} & \multirow{2}{*}{$k_1 = 1$} & \multirow{2}{*}{$k_1 = 4$}  \\
 \textbf{Dimension $(d)$}  & & \\
 \textbf{$\downarrow$}  & & \\
 \midrule
$d = 5$ & $1.71E^{-2} $ & $\bm{1.34E^{-3}} $ \\
&&\\
$d = 10$ &$4.32E^{-2}$ & $\bm{9.4E^{-3}} $ \\ 
&&\\
$d = 50$ &$ 9.47E^{-2}$ & $\bm{6.57E^{-2}} $ \\ 
&&\\
$d = 100$& $1E^{-1}$& $\bm{8.92E^{-2}}$ \\
&&\\
 \bottomrule 
\end{tabular}
\end{subtable}
\begin{subtable}{0.334\textwidth}
		\centering
\caption{Parameter error for GSF}
\label{tab:GSF}
\begin{tabular}{|c|c|c}
\toprule
\rowcolor{gray!20}
\multicolumn{3}{|c|}{\multirow{2}{*}{\textbf{GSF with Rastrigin objective}}}\\[1em]
\midrule
 \textbf{Parameter $(k_1) \; \; \rightarrow$} & \multirow{2}{*}{$k_1=1$} & \multirow{2}{*}{$k_1=4$}   \\
 \textbf{Dimension $(d)$}  & & \\
 \textbf{$\downarrow$}  & & \\
 \midrule
$d=5$ & $3.35E^{-3}$ & $\bm{2.7E^{-4}} $  \\
&&\\
$d=10$ &$ 1.11E^{-2}$ & $\bm{1.14E^{-3}} $  \\ 
&&\\ 
$d=50$ &$1.9E^{-1} $ & $\bm{7.94E^{-3}} $  \\ 
&&\\
$d=100$& $6.4E^{-1} $& $\bm{1.31E^{-2}}$  \\
 \bottomrule
 \rowcolor{gray!20}
\multicolumn{3}{|c}{\multirow{2}{*}{\textbf{GSF with quadratic objective}}}\\[1em]
\midrule
 \textbf{Parameter $(k_1) \; \; \rightarrow$} & \multirow{2}{*}{$k_1 = 1$} & \multirow{2}{*}{$k_1 = 4$}  \\
 \textbf{Dimension $(d)$}  & & \\
 \textbf{$\downarrow$}  & & \\
 \midrule
$d = 5$ & $ 5.89E^{-3}$ & $\bm{2.14E^{-3}} $ \\
&&\\
$d = 10$ &$2.8E^{-2} $ & $\bm{1.34E^{-2}} $ \\ 
&&\\
$d = 50$ &$8.03E^{-2} $ & $\bm{7.4E^{-2}} $ \\ 
&&\\ 
$d = 100$& $\bm{9.03E^{-2}}$& $9.46E^{-2}$ \\
&&\\
 \bottomrule 
\end{tabular}
\end{subtable}
\hspace*{1ex}
\begin{subtable}{0.333\textwidth}
		\centering
\caption{Parameter error for B-GSPSA}
\label{tab:B-GSPSA}
\begin{tabular}{|c|c|c|}
\toprule
\rowcolor{gray!20}
\multicolumn{3}{c|}{\multirow{2}{*}{\textbf{B-GSPSA with Rastrigin objective}}}\\[1em]
\midrule
 \textbf{Parameter $(k_2)\rightarrow$} & \multirow{2}{*}{$k_2=1$} & \multirow{2}{*}{$k_2=2$}   \\
 \textbf{Dimension $(d)$}  & & \\
 \textbf{$\downarrow$}  & & \\
 \midrule
$d=5$ & $5.64E^{-2}$ & $\bm{1.12E^{-9}} $  \\
&&\\
$d=10$ &$ 5.64E^{-2}$ & $\bm{2.47E^{-9}} $  \\ 
&&\\ 
$d=50$ &$5.64E^{-2} $ & $\bm{3.33E^{-4}} $  \\ 
&&\\
$d=100$& $5.63E^{-2} $& $\bm{2E^{-2}}$  \\
 \bottomrule
 \rowcolor{gray!20}
\multicolumn{3}{|c|}{\multirow{2}{*}{\textbf{B-GSPSA with quadratic objective}}}\\[1em]
\midrule
 \textbf{Parameter $(k_2)\rightarrow$} & \multirow{2}{*}{$k_2 = 1$} & \multirow{2}{*}{$k_2 = 2$}  \\
 \textbf{Dimension $(d)$}  & & \\
 \textbf{$\downarrow$}  & & \\
 \midrule
$d = 5$ & $\bm{8.33E^{-4}} $ & $ 1.04E^{-3}$ \\
&&\\
$d = 10$ &$\bm{8.92E^{-3}} $ & $9.14E^{-3} $ \\ 
&&\\
$d = 50$ &$ 6.33E^{-2}$ & $\bm{6.29E^{-2}} $ \\ 
&&\\
$d = 100$& $8.83E^{-2}$& $\bm{8.73E^{-2}}$ \\
&&\\
 \bottomrule 
\end{tabular}

\end{subtable}

\end{table*}
\subsection{Proof of Theorem \ref{thm:nonasymp}}
\label{sec:proof-nonasymp}
\begin{proof}
	We follow the technique from \cite{bhavsar2022nonasymptotic}.
	Since $F$ is $L$-smooth, we have
	\begin{align}
		F \left( \theta(n+1) \right) 	& \leq F \left( \theta(n) \right) + \left\langle \nabla F \left( \theta(n) \right) , \theta(n+1) - \theta(n) \right\rangle \nonumber\\
  &\qquad+ \frac { L } { 2 } \left\| \theta(n+1) - \theta(n) \right\| ^ { 2 } \nonumber \\ 
		& \leq F \left( \theta(n) \right) - a \left\langle \nabla F \left( \theta(n) \right) , \widehat{\mathcal{D}}^{k_1} F(\theta(n)) \right\rangle  \nonumber\\
  &\qquad+ \frac { L } { 2 } a ^ { 2 } \left\| \widehat{\mathcal{D}}^{k_1} F(\theta(n)) \right\| ^ { 2 } \label{eq:take_exp1} .
	\end{align}	
Let $\E_n$ denote the expectation with respect to the sigma field $ \F_{n} $. 
Then, taking expectations in \eqref{eq:take_exp1}, in conjunction with the bounds in Lemma \ref{lemma:gspsa-bias}, we obtain
	\vspace{-.3mm}
	\begin{align}
		& \E_{n} \left[F \left( \theta(n+1) \right)\right] \nonumber \\
		& \leq \E_{n} \left[F \left( \theta(n) \right) \right] - a \left\langle \nabla F \left( \theta(n) \right) , \nabla F \left( \theta(n) \right) + c_1  \delta^{k_1}  \mathbf{1}_{d \times 1} \right\rangle  \nonumber\\
		& \quad
		+ \frac { L } { 2 } a ^ { 2 } \left[ \left\| \mathbb { E }_{n}  \left[ \widehat {\mathcal{D}}^{k_1}F(\theta(n)) \right]  \right\| ^{2} + \frac{c_2}{ \delta^2} \right] \nonumber \\
		&  \leq F \left( \theta(n) \right)
		- a \left\| \nabla F \left( \theta(n) \right) \right\| ^ { 2 } + c_1  \delta^{k_1}  a    \E_{n}  \| \nabla F \left( \theta(n) \right) \|_1  \nonumber \\
		& \quad + \frac { L } { 2 } a ^ { 2 } \left[ \left\| \nabla F \left( \theta(n) \right) \right\| ^ { 2 } \right.\nonumber\\
  &\left.\qquad\qquad+ 2 c_1  \delta^{k_1}  \E_{n}  \| \nabla F \left( \theta(n) \right) \|_1 + {d}c_1^2  \delta^{2k_1} + \frac{c_2}{ \delta^2} \right] \nonumber\\
		& \leq F \left( \theta(n) \right) - \left( a - \frac { L } { 2 } a ^ { 2 } \right)  \left\| \nabla F \left( \theta(n) \right) \right\| ^ { 2 } \nonumber\\
		&\qquad+ c_1  \delta^{k_1} B  \left( a + L a ^ { 2 } \right) + \frac { L } { 2 } a ^ { 2 } \left[ d c_1^2  \delta^{2k_1} + \frac{c_2}{ \delta^2}\right], \label{eq:s11} 
	\end{align}
	where we used the inequality $ - \|\theta\|_1 \leq \sum_{i=1}^{d} \theta(i) $ for any vector $ \theta $ to arrive at \eqref{eq:s11} and $B$ as in \ref{ass: gradeint bound}. For the last inequality, we used (A6).
	From a simple re-arrangement of terms in \eqref{eq:s11}, we obtain
	\begin{align*}
		&a \left\| \nabla F \left( \theta(n) \right) \right\| ^ { 2 }	
		\leq \frac{2}{\left( 2  -  { L }  a \right) } \bigg[ F \left( \theta(n) \right) -  \E_{n} F \left( \theta(n+1) \right) \bigg. \\
		& \quad \bigg.  + c_1  \delta^{k_1} \left( a + L a ^ { 2 } \right)B \bigg] +  \frac{{ L } a ^ { 2 }}{\left( 2  -  { L }  a \right) } \left[ d c_1^2  \delta^{2k_1} + \frac{c_2}{ \delta^2}\right].
	\end{align*}
	Summing over $ n = 1 $ to $ m $, and taking expectations on both sides, we obtain
	\begin{align*}
		&	\sum _ { n = 1 } ^ { m}  a \E_{m}  \left\| \nabla F \left( \theta(n) \right) \right\| ^ { 2 } \\
		& \le  
  \frac{2}{2- L a} \left[F \left( \theta(1) \right)  - \mathbb { E }_{m}  \left[ F \left( \theta (m+1)  \right) \right] \right] \nonumber \\
		&  + 2 \sum _ { n = 1 } ^ { m} c_1  \delta^{k_1} B \left[ \frac{ a + L a ^ { 2 } }{ 2  -  { L }  a  }\right]  + L \sum_{ n = 1 }^{m} \frac{ a ^ { 2 }}{\left( 2  -  { L }  a \right) } \left[ d c_1^2  \delta^{2k_1}  + \frac{c_2}{ \delta^2}\right].
	\end{align*}
	Using $ \mathbb { E }_{m}  \left[ F \left( \theta(n) \right) \right] \ge F(\theta^*) $, we obtain
	\begin{align}
		 &\sum _ { n = 1 } ^ { m }  a \E_{m} \left\| \nabla F \left( \theta(n) \right) \right\| ^ { 2 } \nonumber\\
        &\leq \frac { 2 \left(F(\theta(1)) - F(\theta^*) \right)} { \left( 2-  { L }  a \right)}  
		+ 2 \sum _ { n = 1 } ^ { m} c_1  \delta^{k_1} B \left( \frac{ a + L a ^ { 2 } }{ 2  -  { L }  a  }\right) \nonumber \\
		& \quad + L \sum_{ n = 1 }^{m} \frac{ a ^ { 2 }}{\left( 2  -  { L }  a \right) } \left[ d c_1^2  \delta^{2k_1}  + \frac{c_2}{ \delta^2}\right].\nonumber
	\end{align}
	Since $\theta(R)$ is picked uniformly at random from $\{\theta (1),\ldots,\theta (m)\}$, we obtain
	\begin{align}
		&\mathbb { E } \left[ \left\| \nabla F \left( \theta(R) \right) \right\| ^ { 2 } \right] \nonumber\\
		&\le \frac{1 }{m a} \left[\frac { 2 (F(\theta(1)) - F(\theta^*))} { \left( 2-  { L }  a \right)}   + 2 B \sum _ { n = 1 } ^ { m} c_1  \delta^{k_1} \left( \frac{ a + L a ^ { 2 } }{ 2  -  { L }  a  }\right)  \right. \nonumber \\
		&\quad\qquad \left.    + L \sum_{ n = 1 }^{m} \frac{ a ^ { 2 }}{\left( 2  -  { L }  a \right) } \left[ d c_1^2  \delta^{2k_1}  + \frac{c_2}{ \delta^2}\right] \right]. \label{eq:prop_biased_bound}
	\end{align}
Using \eqref{eq:biased_sp_par}, we have
	\begin{align}
		&	\mathbb { E } \left[ \left\| \nabla F \left( \theta(R) \right) \right\| ^ { 2 } \right] \nonumber \\
		& \le  \frac{1}{{ m }a } \left[ { 2 (F(\theta(1)) - F(\theta^*))} + 4m a B c_1\delta^{k_1} \right.\\
  &\left. \qquad\qquad+  { L m a ^ { 2 }}\left[ d c_1^2\delta^{2k_1}  + \frac{c_2}{\delta^2}\right] \right] \label{eq:using_1l}\\
		& \le \frac{ 2 (F(\theta(1)) - F(\theta^*))}{{ m }} max\bigg\{L, {m^{\frac{k_1+2}{2k_1+2}}}\bigg\}    + 4 B \left(\frac{c_1 }{m^{\frac{k_1}{2k_1+2}}}  \right)\nonumber\\
  &\qquad+   { L }\left[ \frac{d c_1^2}{m^{\frac{2k_1}{2k_1+2}}}
		 + \frac{c_2}{m^{-\frac{2}{2k_1+2}}}\right] \frac{1}{m^{\frac{k_1+2}{2k_1+2}}} \label{eq:using_defs}.
	\end{align}
	In the above, the inequality \eqref{eq:using_1l} follows by using the fact that $ a \leq 1/L $, while the inequality \eqref{eq:using_defs} uses the settings for $a,\delta$ in \eqref{eq:biased_sp_par}. The main claim follows by rearranging terms in \eqref{eq:using_defs}.
\end{proof}

\section{Simulation Experiments}
\label{sec:experiments}
We present in this section our detailed simulation results where we implement and compare the performance of the various algorithms presented in this paper.
We begin by first presenting the simulation setup.

\subsection{Simulation Setup}
In our experiments, we consider the following $d$-dimensional minimization problem:
\begin{equation}
    \min_{\theta} F(\theta),
\end{equation}
given noisy observations $f(\theta,\xi) = F(\theta) +\xi$.
Here, for any  $\theta$, the noise term $\xi=[\theta^T,1]z$, where $z$ is a $(d+1)$-dimensional multivariate Gaussian distribution $\N$ with mean $0$ and covariance matrix $\sigma^2{\cal{I}}_{d+1}$. In our experiments, we set $\sigma$ = $0.001$ and $0.1$ respectively.

We consider two well known objective functions namely quadratic and Rastrigin with $d = 5, 10, 50$ and $100$ to measure the performance of our proposed methods. We implement various algorithms within our (proposed) broad family of GSPGS algorithms for different values of $k_1$. Note that, $ k_1 = 1$ corresponds to vanilla one-sided SPSA/RDSA/SF. We also implement our B-GSPSA method for different values of  $k_2$, where $k_2 = 1$ represents the vanilla two-sided SPSA. We now describe the two objective functions used.

\textbf{Quadratic Function:} Let $A$ be a $d \times d$ upper triangular matrix with entries $\frac{1}{d}$ and $b$ be a $d$-dimensional vector of ones. Then, the quadratic objective function is defined as follows:
\begin{equation}
\label{eq:quadratic}
    f(\theta,\xi) = \theta^T A \theta + b^T \theta + \xi.
\end{equation}
In our experiment, we use $d = 5, 10, 50$ and $100$, respectively. For $d = 10$, the optimal point $\theta^*$ is a 10-dimensional vector with entries $-0.90909091$.

\textbf{Rastrigin Function:} 
The $d$-dimensional Rastrigin objective function is defined as follows:

\begin{equation}
\label{eq:Rastrigin}
    f(\theta,\xi) = 10d + \sum_{i = 1}^{d}[\theta_i^2 - 10 \cos(2\pi \theta_i)] +\xi,
\end{equation}
where $\xi$ is the same noise r.v. as before. Here optimal $\theta^*$ is the $d$-dimensional vector of zeros and initial $\theta_0$ is set to a $d$-dimensional vector of twos. 

To measure the performance of our proposed algorithm for a given simulation budget, we consider the following metric:
\begin{equation}
    \text{Parameter error} = \frac{||\theta_{\tau} - \theta^*||^2}{|| \theta_0 - \theta^*||^2},
    \label{eq:param_error}
\end{equation}
where $\tau$ is the iteration index when parameter $\theta$ is last updated at the end of the simulation. The parameter error is measured as the ratio of the squared distance between the final update of the parameter $\theta$ and the optimal parameter $\theta^*$  to the squared distance between the initial parameter value and the optimal parameter $\theta^*$. 

\begin{figure*}
	\begin{tabular}{cc}
		\begin{subfigure}{0.5\textwidth} 
			\tabl{c}{\scalebox{0.75}{\begin{tikzpicture}
						\begin{axis}[
							ybar={2pt},
							legend style={at={(0.8,0.9)},anchor=north,legend columns=-1},
							legend image code/.code={\path[fill=white,white] (-2mm,-2mm) rectangle
								(-3mm,2mm); \path[fill=white,white] (-2mm,-2mm) rectangle (2mm,-3mm); \draw
								(-2mm,-2mm) rectangle (2mm,2mm);},
							ylabel={\bf Parameter error},
							xlabel={\bf Truncation parameter},
							symbolic x coords={0, 1, 2, 3, 4, 5},
							xmin={0},
							xmax={5},
							xtick=data,
							ytick align=outside,
							xticklabels={{$k_1 = 1$,$k_1=2$,$k_1=3$,$k_1=4$}},
							xticklabel style={align=center},
							bar width=14pt,
							nodes near coords,
							grid,
							grid style={gray!30},
							width=11cm,
							height=7.25cm,
							]
							\addplot   coordinates { (1,0.05638624152862355)(2,0.052998176882583056)(3,0.03315527659499533)(4,0.014168105885141852)}; 
							\addlegendentry{GSPSA}
						\end{axis}
				\end{tikzpicture}}\\[1ex]}
			\caption{ $d = 5$}
			\label{fig:rastrigin5_1}
		\end{subfigure}
		&
		\begin{subfigure}{0.5\textwidth}
			\tabl{c}{\scalebox{0.75}{\begin{tikzpicture}
						\begin{axis}[
							ybar={2pt},
							legend style={at={(0.8,0.9)},anchor=north,legend columns=-1},
							legend image code/.code={\path[fill=white,white] (-2mm,-2mm) rectangle
								(-3mm,2mm); \path[fill=white,white] (-2mm,-2mm) rectangle (2mm,-3mm); \draw
								(-2mm,-2mm) rectangle (2mm,2mm);},
							ylabel={\bf Parameter error},
							xlabel={\bf Truncation parameter},
							symbolic x coords={0, 1, 2, 3, 4, 5},
							xmin={0},
							xmax={5},
							y tick label style={
								/pgf/number format/.cd,
								fixed,
								fixed zerofill,
								precision=2,
								/tikz/.cd
							},
							xtick=data,
							ytick align=outside,
							xticklabels={{$k_1 = 1$,$k_1=2$,$k_1=3$,$k_1=4$}},
							xticklabel style={align=center},
							bar width=14pt,
							nodes near coords,
							grid,
							grid style={gray!30},
							width=11cm,
							height=7.25cm,
							]
							\addplot   coordinates { (1,0.056319866737310265)(2,0.05300952783308738)(3,0.02972567955701478)(4,0.015375558270423353)}; 
							\addlegendentry{GSPSA}
						\end{axis}
				\end{tikzpicture}}\\[1ex]}
			\caption{ $d = 10$}
			\label{fig:rastrigin10_1}
		\end{subfigure}\\
		\begin{subfigure}{0.5\textwidth}
			\tabl{c}{\scalebox{0.75}{\begin{tikzpicture}
						\begin{axis}[
							ybar={2pt},
							legend style={at={(0.8,0.9)},anchor=north,legend columns=-1},
							legend image code/.code={\path[fill=white,white] (-2mm,-2mm) rectangle
								(-3mm,2mm); \path[fill=white,white] (-2mm,-2mm) rectangle (2mm,-3mm); \draw
								(-2mm,-2mm) rectangle (2mm,2mm);},
							ylabel={\bf Parameter error},
							xlabel={\bf Truncation parameter},
							symbolic x coords={0, 1, 2, 3, 4, 5},
							xmin={0},
							xmax={5},
							y tick label style={
								/pgf/number format/.cd,
								fixed,
								fixed zerofill,
								precision=2,
								/tikz/.cd
							},
							xtick=data,
							ytick align=outside,
							xticklabels={{$k_1 = 1$,$k_1=2$,$k_1=3$,$k_1=4$}},
							xticklabel style={align=center},
							bar width=14pt,
							nodes near coords,
							grid,
							grid style={gray!30},
							width=11cm,
							height=7.25cm,
							]
							\addplot   coordinates { (1,0.056409970213939564)(2,0.05298931904568991)(3,0.03411334653962775)(4,0.010464818512111928)};
							\addlegendentry{GSPSA}
						\end{axis}
				\end{tikzpicture}}\\[1ex]}
			\caption{ $d = 50$}
			\label{fig:rastrigin50_1}
		\end{subfigure} 
		&
		\begin{subfigure}{0.5\textwidth}
			\tabl{c}{\scalebox{0.75}{\begin{tikzpicture}
						\begin{axis}[
							ybar={2pt},
							legend style={at={(0.8,0.9)},anchor=north,legend columns=-1},
							legend image code/.code={\path[fill=white,white] (-2mm,-2mm) rectangle
								(-3mm,2mm); \path[fill=white,white] (-2mm,-2mm) rectangle (2mm,-3mm); \draw
								(-2mm,-2mm) rectangle (2mm,2mm);},
							ylabel={\bf Parameter error},
							xlabel={\bf Truncation parameter},
							symbolic x coords={0, 1, 2, 3, 4, 5},
							xmin={0},
							xmax={5},
							y tick label style={
								/pgf/number format/.cd,
								fixed,
								fixed zerofill,
								precision=2,
								/tikz/.cd
							},
							xtick=data,
							ytick align=outside,
							xticklabels={{$k_1 = 1$,$k_1=2$,$k_1=3$,$k_1=4$}},
							xticklabel style={align=center},
							bar width=14pt,
							nodes near coords,
							grid,
							grid style={gray!30},
							width=11cm,
							height=7.25cm,
							]
							\addplot   coordinates {(1,0.05691517199687485)(2,0.052950174146123116)(3,0.032169927920595975)(4,0.009703615423722004)};
							\addlegendentry{GSPSA}
						\end{axis}
				\end{tikzpicture}}\\[1ex]}
			\caption{ $d = 100$}
			\label{fig:rastrigin100_1}
			
		\end{subfigure}
	\end{tabular}
	\caption{Parameter error for GSPSA algorithm with different values of measurements $m$ and dimension $d$ for Rastrigin objective function for $\sigma = 0.1$. The results are averages over $20$ independent replications. }
	\label{fig:gspsa_01}
\end{figure*}
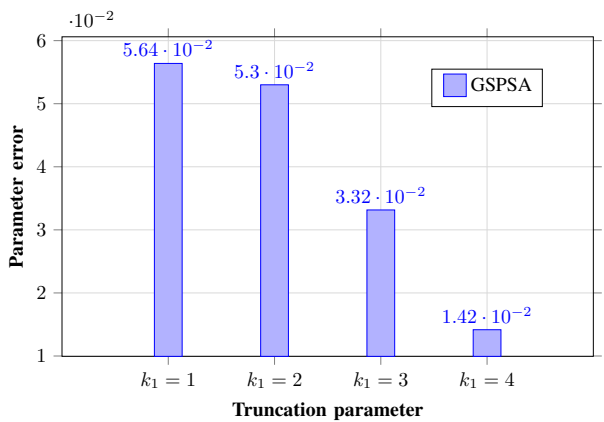
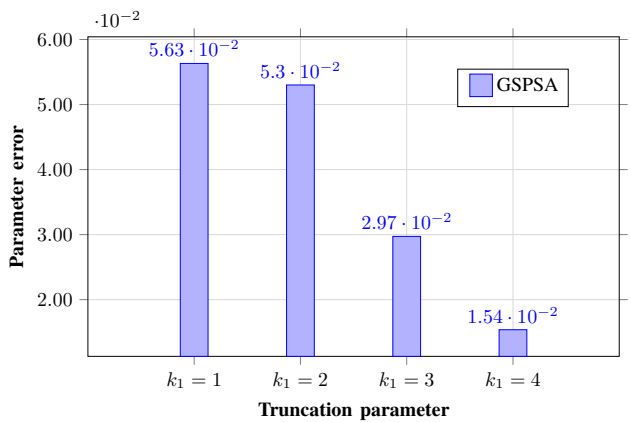
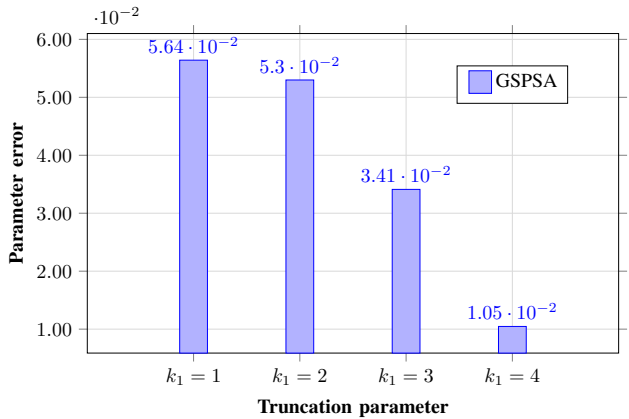
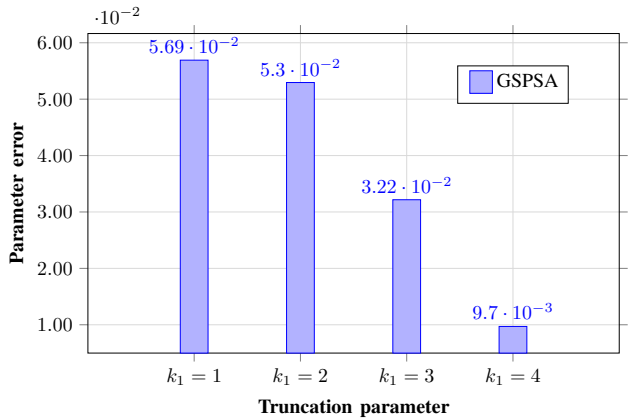
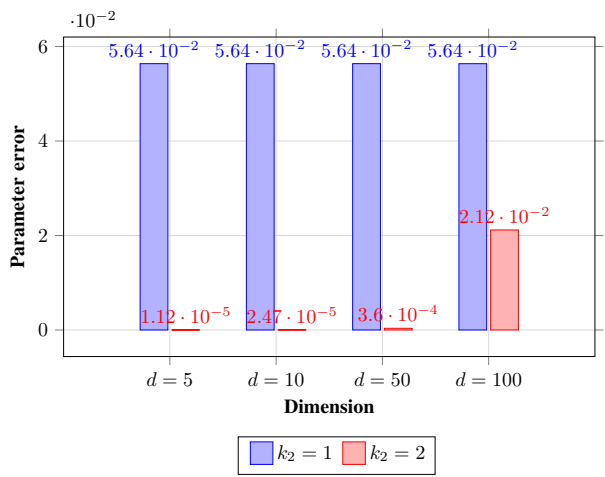
\begin{figure}
	\centering
	\tabl{c}{\scalebox{0.75}{\begin{tikzpicture}
				\begin{axis}[
					ybar={2pt},
					legend style={at={(0.7,-0.25)},legend columns=-1},
					legend image code/.code={\path[fill=white,white] (-2mm,-2mm) rectangle
						(-3mm,2mm); \path[fill=white,white] (-2mm,-2mm) rectangle (2mm,-3mm); \draw
						(-2mm,-2mm) rectangle (2mm,2mm);},
					ylabel={\bf Parameter error},
					xlabel={\bf Dimension},
					symbolic x coords={0, 1, 2, 3, 4, 5},
					xmin={0},
					xmax={5},
					xtick=data,
					ytick align=outside,
					xticklabels={{$d = 5$,$d=10$,$d=50$,$d=100$}},
					xticklabel style={align=center},
					bar width=14pt,
					nodes near coords,
					grid,
					grid style={gray!30},
					width=11cm,
					height=7.25cm,
					]
					\addplot   coordinates { (1,0.05637251299133461)(2,0.05637479300133478)(3,0.05637567089672997) (4,0.056368761177056345)};
					\addplot   coordinates { (1,1.1203763415022605e-05)(2,2.470035728896026e-05)(3,0.0003598070344015937) (4,0.021173414421788748)};
					\legend{$k_2=1$,$k_2=2$}
				\end{axis}
		\end{tikzpicture}}\\[1ex]}
	\caption{Parameter error for B-GSPSA for different values of dimension $d$ with a Rastrigin objective function ($\sigma = 0.1$). Note that $k_2=1$ corresponds to regular balanced SPSA.}
	\label{fig:rastrigin_balanced_1}
\end{figure}

We implement the following algorithms:
\begin{description}
	\item[GSPSA] \ \\This corresponds to Algorithm \ref{alg:GSPSA}, which employs a vector of symmetric $\pm1$-valued Bernoulli random variables for the random perturbations;
	\item[GSF] \ \\This is similar to Algorithm \ref{alg:GSPSA}, except that the gradient estimate is formed using \eqref{eq:gspg-noisy} with $U \sim \N (0, I_d)$ and $V=U$.
	\item[GRDSA]\ \\This is also similar to Algorithm \ref{alg:GSPSA}, but with random perturbations that are a vector of independent uniform random variables \cite{prashanth2017rdsa}. 
	\item[B-GSPSA]\ \\ This corresponds to the algorithm with balanced estimates described in Section \ref{sec: balanced_estimator12}, which employs the same random perturbations as GSPSA.
\end{description}

We test the performance of the algorithms mentioned above on both quadratic and Rastrigin objective functions for different values of $d$.
In our experiments, we use a simulation budget (i.e., number of function measurements) of $2\times 10^{5}$. The number of iterations an algorithm runs depends on the parameter $k_1$ or $k_2$, depending on whether the algorithm employs a one-sided or balanced gradient estimator.
For instance, the number of iterations $\tau = \text{budget}/(k_1+1)$ for $(k_1 +1)$-measurement GSPSA. Thus, $\tau$ is different for different algorithms as it depends on the form of the estimators used.
The step sizes $a(n)$ and perturbation parameter $\delta (n)$ used in the implementation are specified in Tables \ref{tab:hyperparameter} and \ref{tab:hyperparameter_balance}. The constants used in these parameters were found to give the best performance among the various parameter choices that we investigated.
\begin{table}[h!]
	\centering
	\caption{Step size and perturbation parameter settings for various GSPGS algorithms}
	\begin{tabular}{|c|c|c|c|}
		
		\hline
		\text{Function} & Method &  $\delta (n)$ & $a (n)$\\
		\hline 
		& GSPSA & $7.9/n^{0.101}$ & $1/(n+50)$\\
		Quadratic& GRDSA & $26.8/n^{0.101}$ &$1/(n+65)$ \\
		& GSF & $5.9/n^{0.101}$ &$1/(n+50)$ \\ \hline
		& GSPSA & $2.9/n^{0.101}$&$3/(n+50)$ \\
		Rastrigin & GRDSA &$26.4/n^{0.101}$ & $1/(n + 50)$\\
		& GSF & $26.8/n^{0.101}$& $1/(n+50)$\\
		\hline
	\end{tabular}
	
	\label{tab:hyperparameter}
\end{table}

\begin{table}[h!]
    \centering
    \caption{ Step size and perturbation parameter  settings for B-GSPSA}
\begin{tabular}{|c|c|c|c|}
    
    \hline
        \text{Function}& \text{Method} & $\delta (n)$ & $a (n)$\\
        \hline
         Quadratic & B-GSPSA & $26.8/n^{0.101}$ &$1/(n+65)$ \\
         Rastrigin  & B-GSPSA & $2.9/n^{0.101}$&$2/(n+20)$ \\
        \hline
    \end{tabular}
    
    \label{tab:hyperparameter_balance}
\end{table}

\subsection{Results} 
Table \ref{tab:GSPSA} presents the parameter error for GSPSA with noise variance $\sigma=0.001$ under the two different objective functions.
Figure \ref{fig:gspsa_01} presents similar results, but with the noise variance $\sigma=0.1$ for the Rastrigin objective function. In particular, Figures \ref{fig:rastrigin5_1}, \ref{fig:rastrigin10_1}, \ref{fig:rastrigin50_1} and \ref{fig:rastrigin100_1} present the parameter error for dimensions $5,10, 50$ and $100$, respectively. All the reported results are averages over $20$ independent replications.  
From these results, it is apparent that increasing the number of measurements with a fixed simulation budget, leads to improved results for GSPSA. This conclusion holds even as we vary the parameter dimension in the set $\{5,10,50,100\}$.  

Table \ref{tab:nonGSPSAresults} presents the parameter error for three different algorithms under two different objective functions for the case when the noise variance $\sigma=0.001$. 
In particular,  Tables \ref{tab:GRDSA}, \ref{tab:GSF} and \ref{tab:B-GSPSA} present the parameter error for GRDSA, GSF and B-GSPSA algorithms, respectively, with the parameter dimension varying in the set $\{5,10,50,100\}$. 
Figure \ref{fig:rastrigin_balanced_1} presents the parameter error for B-GSPSA with noise variance $\sigma=0.1$. From these results, we observe that, for the same simulation budget, increasing the number of function measurements leads to improved performance, as in the case of GSPSA. Further, B-GSPSA outperforms the other algorithms, as it results in a significantly lower parameter error while using the same simulation budget. This can be attributed to the better bias guarantee, as shown in Lemma \ref{lemma:gspsa-bias-balanced}, for B-GSPSA that also does so with a lower number of function measurements, as compared to GSPSA and the other algorithms.

\section{Conclusions}
\label{conclusions}
We presented in this paper a family of generalized simultaneous perturbation-based gradient search (GSPGS) estimators. These estimators differ from one another in 
(a) the number of function measurements per iteration that each one of them requires and (b) the form of the perturbation distribution used. In particular, we presented generalized unbalanced and balanced SPSA, SF and RDSA gradient estimators. 

We showed analytically that estimators within any specified class requiring more number of function measurements result in a lower estimation bias.
We presented a detailed analysis of both the asymptotic and non-asymptotic convergence of the generalized algorithms. 
Finally, we presented the results of detailed experiments involving various GSPGS estimators and observed that the B-GSPGS estimators perform the best as they provide better accuracy with less number of function measurements. Moreover, as suggested by the theory, estimators requiring more number of function measurements in general result in better performance.

As future work, it would be interesting to study the numerical performance of these algorithms on real-life engineering applications. An orthogonal direction is to extend the generalized SPSA idea to stochastic Newton methods, which require the Hessian matrix to be estimated. 

\section*{Acknowledgements}
SB was supported in part by the J.C.Bose National Fellowship of SERB, Project No.~DFTM/02/3125/M/04/AIR-04 from DRDO under DIA-RCOE, a project from DST under the ICPS program as well as the RBCCPS, IISc. LAP would like to thank Sanjay Bhat for his help with some algebraic simplifications.
\bibliographystyle{IEEEtran}
\bibliography{IEEEabrv,tac_refs}

\end{document}